\documentclass[11pt]{article}

\newcommand{\arxiv}[1]{#1}
\newcommand{\colt}[1]{}

\usepackage{arxiv_curriculum}
\usepackage{times}

\title{Learning to Reason with Curriculum I:\\[2pt] Provable Benefits of Autocurriculum}

\usepackage{inconsolata}

\renewcommand{\Pr}{\mathrm{Pr}}

\newcommand{\abort}{\texttt{ABORT}}

\newcommand{\tabort}{\widetilde{\texttt{ABORT}}}
\newcommand{\err}{\texttt{err}}

\newcommand{\rank}{\mathrm{rank}}

\newcommand{\wrank}{\widehat{\mathrm{rank}}}
\newcommand{\tw}{\widetilde{w}}
\newcommand{\ww}{\widehat{w}}
\newcommand{\tp}{\widetilde{p}}

\newcommand{\trho}{\widetilde{\rho}}

\newcommand{\tPhi}{\widetilde{\Phi}}

\usepackage{cases}

\newcommand{\acc}{\texttt{Acc}}
\newcommand{\wacc}{\widehat{\texttt{Acc}}}

\newcommand{\talpha}{\widetilde{\alpha}}
\newcommand{\tbeta}{\widetilde{\beta}}

\newcommand{\Vlearn}{\texttt{RLFineTune}}
\newcommand{\NTP}{\texttt{NTP}}

\newcommand{\fhat}{\widehat{f}}
\newcommand{\fhatcons}{\widehat{f}_{\texttt{cons}}}
\newcommand{\ghat}{\widehat{g}}

\newcommand{\CoT}{\texttt{CoT}}
\newcommand{\SFT}{\texttt{SFT}}
\newcommand{\RL}{\texttt{RL}}

\newcommand{\Genboost}{\texttt{AutoTune}}
\newcommand{\Genbooststoch}{\texttt{AutoTune}_{\texttt{stoch}}}

\newcommand{\VCdim}{\mathrm{VC}}
\newcommand{\Ndim}{\mathrm{Ndim}}

\newcommand{\Alg}{\texttt{Alg}}
\newcommand{\calOtilde}{\widetilde{\mathcal{O}}}

\newcommand{\piref}{\pi_{\texttt{ref}}}
\newcommand{\pihat}{\widehat{\pi}}

\newcommand{\Cseq}{C_{\texttt{seq}}}

\newcommand{\nCoT}{n_{\texttt{CoT}}}

\newcommand{\ncomp}{n_{\texttt{comp}}}
\newcommand{\nprompt}{n_{\texttt{prompt}}}
\newcommand{\Dprompt}[1]{D_{\texttt{prompt}}^{#1}}

\newcommand{\Vboost}{\texttt{AutoTune.RL}}
\newcommand{\Dout}[1]{D_{\texttt{out}}^{#1}}
\newcommand{\maj}{\texttt{Plu}}

\newcommand{\subsample}{\texttt{Sample}}

\newcommand{\unif}{\operatorname{Unif}}




\tcbuselibrary{skins,breakable}

\newtcolorbox[
  auto counter,
  number within=section,
  number freestyle={\noexpand\thesection.\noexpand\arabic{\tcbcounter}~\noexpand\mytitle},
  crefname={Meta-algorithm}{Meta-algorithms}
]{metaalgorithm}[2][]{
  breakable,
  enhanced,
  colback=white,
  colframe=black,
  coltitle=black,
  fonttitle=\bfseries,
  code={\def\mytitle{#2}},
  title=Meta-algorithm \thetcbcounter,
  boxrule=0.8pt,
  arc=2pt,
  left=8pt,right=8pt,
  top=6pt,bottom=6pt,
  attach boxed title to top center={yshift=-2mm},
  boxed title style={
    colback=white,
    colframe=black,
    boxrule=0.8pt,
    sharp corners
  },
  #1
}

\crefname{claim}{claim}{Claims}

\newtcolorbox[
  auto counter,
  number within=section,
  number freestyle={\noexpand\thesection.\noexpand\arabic{\tcbcounter}~\noexpand\mytitle},
  crefname={Induction}{Inductions}
]{induction}[2][]{
  breakable,
  enhanced,
  colback=white,
  colframe=black,
  coltitle=black,
  fonttitle=\bfseries,
  code={\def\mytitle{#2}},
  title=$\text{Induction}_{\bm{#2}} (\delta)$,
  boxrule=0.8pt,
  arc=2pt,
  left=8pt,right=8pt,
  top=6pt,bottom=6pt,
  attach boxed title to top left={yshift=-2mm},
  boxed title style={
    colback=white,
    colframe=black,
    boxrule=0.8pt,
    sharp corners
  },
  #1
}

\SetCommentSty{mycommentsty}

\SetKwComment{cmt}{{\color{blue}$\blacktriangleright$}\ }{}
\SetKwComment{cmttwo}{}{}
\usepackage{changepage}

\makeatletter
\renewcommand{\algocf@printnl}[1]{%
  \ifthenelse{\boolean{algocf@leftlinenumber}}{%
    \skiplinenumber=\skiptotal
    \strut\raisebox{0pt}{\llap{\NlSty{#1}\kern\skiplinenumber}}\ignorespaces%
  }{%
    \sbox\algocf@nlbox{\NlSty{#1}}%
    \skiplinenumber=\hsize
    \advance\skiplinenumber by-\skiptext
    \advance\skiplinenumber by\algomargin
    \advance\skiplinenumber by.3em
    \advance\skiplinenumber by-\wd\algocf@nlbox%
    \advance\skiplinenumber by-\algocf@skipuntil%
    \strut\raisebox{0pt}{\rlap{\kern\skiplinenumber\NlSty{#1\ignorespaces}}}\ignorespaces%
  }%
}%
\makeatother

\newlength{\algonuminset}
\usepackage{tkz-euclide}

\makeatletter
\let\oldnl\nl
\newcommand{\nonl}{\renewcommand{\nl}{\let\nl\oldnl}}
\makeatother

\usepackage{color-edits}
\addauthor{df}{ForestGreen}
\addauthor{nv}{Red}
\addauthor{ak}{BurntOrange}

\newcommand{\loose}{\looseness=-1}

\let\olddfedit\dfedit
\renewcommand{\dfedit}[2]{\olddfedit{#2}}

\author{%
  Nived Rajaraman\thanks{Microsoft Research. \texttt{nrajaraman@microsoft.com}} \and
  Audrey Huang\thanks{University of Illinois Urbana Champaign. \texttt{audreyh5@illinois.edu}} \and
  Miro Dudik\thanks{Microsoft Research. \texttt{mdudik@microsoft.com}} \and
  Robert Schapire\thanks{Microsoft Research. \texttt{schapire@microsoft.com}} \and
  Dylan J.\ Foster\thanks{Microsoft Research. \texttt{dylanfoster@microsoft.com}} \and
  Akshay Krishnamurthy\thanks{Microsoft Research. \texttt{akshaykr@microsoft.com}}
}

\begin{document}

\maketitle

\begin{abstract}
\noindent Chain-of-thought reasoning, where language models expend additional computation by producing thinking tokens prior to final responses, has driven significant advances in model capabilities. However, training these reasoning models is extremely costly in terms of both data and compute, as it involves collecting long traces of reasoning behavior from humans or synthetic generators and further post-training the model via reinforcement learning. Are these costs fundamental, or can they be reduced through better algorithmic design? We show that \textit{autocurriculum}---where the model uses its own performance to decide which problems to focus training on---provably improves upon standard training recipes for both supervised fine-tuning (SFT) and reinforcement learning (RL). For SFT, we show that
autocurriculum requires \textit{exponentially} fewer reasoning demonstrations than non-adaptive fine-tuning~\citep{joshi2025theory}, by focusing teacher supervision on prompts where the current model struggles. For RL fine-tuning, autocurriculum \textit{decouples} the computational cost from the quality of the reference model, reducing the latter to a burn-in cost that is nearly independent of the target accuracy. These improvements arise purely from adaptive data selection, drawing on classical techniques from boosting \citep{freund1997decision} and learning from counterexamples \citep{angluin1987learning}, and requiring no assumption on the distribution or difficulty of prompts.
\end{abstract}


\section{Introduction}

Recent advances in language models have demonstrated that performance on complex reasoning tasks can be substantially improved by increasing inference-time computation. In particular, multi-step chain-of-thought (CoT) reasoning~\citep{wei2022Chain} enables models to solve hard tasks by generating long intermediate reasoning traces before producing the final answer. Reasoning models are often trained to elicit this behavior through a combination of supervised fine-tuning and reinforcement learning~\citep{guo2025deepseek}. Supervised fine-tuning (SFT) on long CoT data has been one of the largest drivers of reasoning capabilities in domains like code-generation or mathematical reasoning \citep{jaech2024openai,hu2025open}, but collecting high-quality data for supervision relies heavily on strong teacher models or human labelers, requiring massive and concerted efforts~\citep{guha2025openthoughts}. 
On the other hand, reinforcement learning fine-tuning (RL) does not require costly CoT data, but is computationally expensive, as it requires generating a large number of reasoning traces from the model being trained~\citep{liu2025prorl}. To push beyond the capabilities of current frontier models, it is essential to reduce these statistical and computational costs.

\medskip
\noindent 
A promising approach to this challenge is \textit{curriculum design}, where reasoning problems of varied difficulty are used to guide the model toward solving progressively harder problems. Curricula can be hand-designed by humans based on intuition about problem difficulty, automatically designed by the model itself, adapting to model capabilities over the course of training, or a combination of both. The latter automatic or \textit{autocurriculum approaches} have recently been incorporated into large-scale RL systems as a mechanism to improve compute efficiency and stabilize training \citep{yu2025dapo,khatri2025art}, and also to gradually scale problem difficulty \citep{zeng2025rlve}. Theoretically, however, the algorithmic, statistical, and computational aspects of autocurriculum are poorly understood, particularly as they pertain to LLM reasoning. With this as motivation, we ask:\loose

\begin{center}
\textit{How should reasoning models design their own  curricula? \\What are the statistical and computational benefits of autocurriculum for reasoning?}
\end{center}

\noindent To address these questions, we study autocurriculum in a theoretical framework for learning with autoregressive models~\citep{joshi2025theory}, encompassing both supervised fine-tuning (SFT) and reinforcement learning with verifiable rewards (RLVR). In both settings, the learner has access to a class of models that generate reasoning traces token-by-token, and an \textit{outcome verifier} (e.g., a unit test for code, or an answer checker for math) that checks whether the final answer is correct.

\subsection{Contributions}

We establish that autocurriculum, using the verifier to adaptively select which prompts to focus training on, provably and dramatically reduces the cost of training, yielding up to \textit{exponential} improvements over non-adaptive approaches.

\begin{figure}
    \centering
    \includegraphics[width=0.85\linewidth]{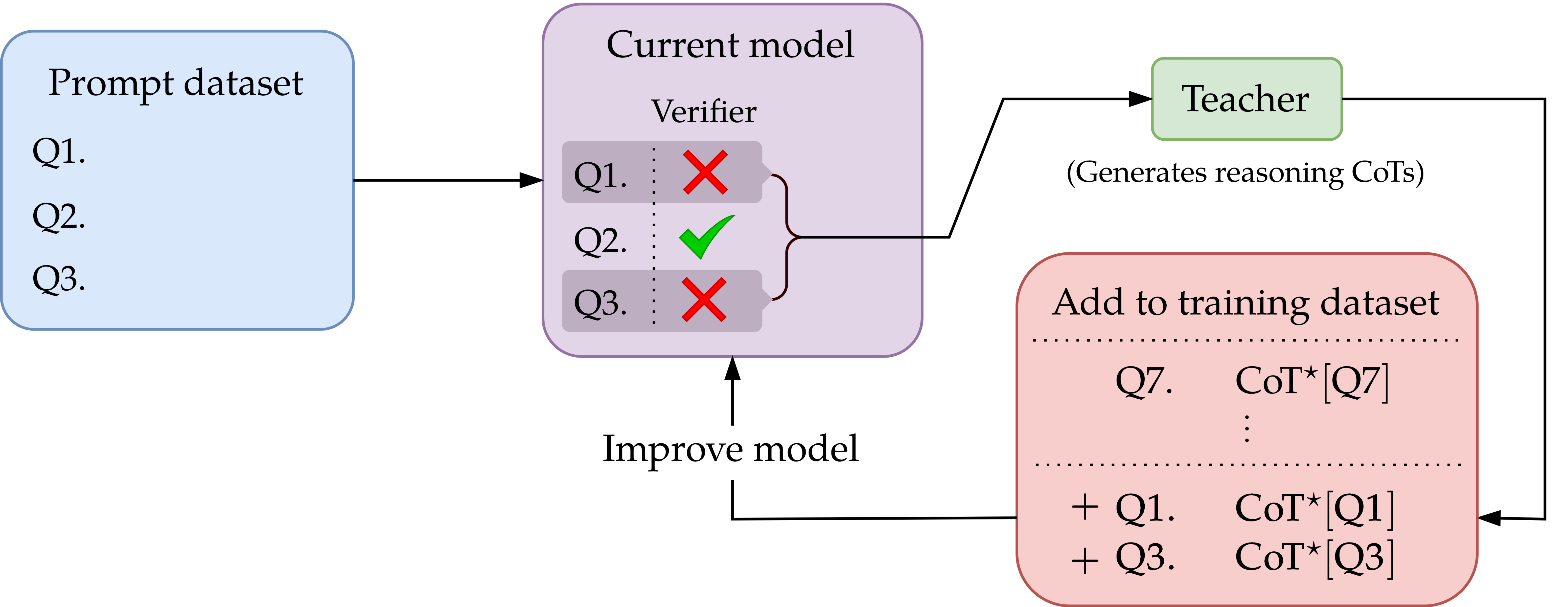}
    \caption{An example of autocurriculum for supervised fine-tuning: the learner chooses which prompts to receive teacher CoTs on, based on its accuracy. In each iteration, the learner's model is updated to digest the new supervision and improve its accuracy.}
    \label{fig:autocurriculum_distillation}
\end{figure}

\paragraph{\Cref{sec:distill}: Supervised fine-tuning.}

In this setting, the learner has interactive access to an expert that generates perfect reasoning traces, capturing SFT and distillation scenarios \citep{shao2024deepseekmath,abdin2025phi,olmo2025olmo}. We design $\Genboost$ (\underline{Auto}curriculum Fine-\underline{Tun}ing), an algorithm inspired by classical boosting~\citep{Freund,schapire2013boosting} and learning from counterexamples~\citep{angluin1987learning,NEURIPS2021_ae06fbdc,pmlr-v76-angluin17a} that iteratively queries for reasoning traces on prompts where the current model errs. We show that this learner is \textit{exponentially} more sample-efficient than the non-adaptive approach of~\citet{joshi2025theory}, reducing the number of reasoning demonstrations from $\widetilde{\Theta}(d/\varepsilon)$ to $\widetilde{O}(d)$, where $d$ measures the complexity of the model class and $1-\varepsilon$ is the target accuracy. In other words, the number of teacher demonstrations becomes nearly independent of the target accuracy. Perhaps surprisingly, this holds without distributional assumptions on the prompt space or hypothesis class---unlike classical active learning, where logarithmic label complexity requires restrictive structural conditions.

\paragraph{\Cref{sec:partial-coverage}: Fine-tuning a reference model.}
Our second setting captures reinforcement learning with verifiable rewards (RLVR)~\citep{lambert2024tulu}. Here, the learner starts with a reference model and must guide it toward higher accuracy through interaction with the outcome verifier. Under a natural coverage assumption on the reference model~\citep{zhu2023principled,song2024importance}, we show that autocurriculum \textit{decouples} the dependence on coverage from the target accuracy, reducing the computational cost from $\widetilde{O}(d \Cseq / \varepsilon)$ to $\widetilde{O}(d \Cseq + d/\varepsilon)$, where $\Cseq$ measures how well the reference model covers the correct reasoning traces. 
The $\widetilde{O}(d\Cseq)$ term reflects a coverage dependent startup cost that grows negligibly with the target accuracy; beyond this burn-in, the $\widetilde{O}(d/\varepsilon)$ cost for improving accuracy matches the cost of a reference model with perfect coverage.


\medskip
\noindent 
In both settings, no assumption on the distribution or difficulty of prompts is required: the autocurriculum emerges entirely from the model's own training dynamics, complementing the more widely studied data-centric approaches in modern pipelines.

\begin{table}[t]
    \centering
    \begin{tabular}{|c|c|c|}
    \hline
    \textbf{Setting} & \textbf{No curriculum} & \textbf{Autocurriculum} \\
    \hline
    \multirowcell{2}{SFT} & \multirowcell{2}{$\nprompt = \nCoT \sim \frac{\log(|\Pi|)}{\varepsilon}$} & \multirowcell{2}{$(\nprompt,\nCoT) \sim \big( \frac{\log(|\Pi|)}{\varepsilon}, {\color{green!70!red!90!black}\log(|\Pi|) \log(1/\varepsilon)} \big)$}\\
    & & \\\hline\hline
    \multirowcell{2}{RLVR} & \multirowcell{2}{$\ncomp \sim \frac{\Cseq \log (|\Pi|)}{\varepsilon}$} & \multirowcell{2}{$\ncomp \sim {\color{green!70!red!90!black} \Cseq \log (|\Pi|) \log(1/\varepsilon) + \frac{\log (|\Pi|)}{\varepsilon}}$}\\
     & & \\\hline
    \end{tabular}
    \caption{``$\sim$'' refers to equivalence up to $\polylog(1/\varepsilon,1/\delta,d,\Cseq)$ factors, $\Pi$ is the class of next-token models, and $\Cseq$ is the sequence-level coverage of the reference model in the RLVR setting (\Cref{def:seq-cov}). Entries in the table describe the complexity of learning a model \smash{$\pihat$} such that \smash{$\acc_{\rho} \big( \pihat \big) \ge 1-\varepsilon$} when $\Pi$ is finite and deterministic. $\nprompt$ denotes the number of prompts, $\nCoT \le \nprompt$ is the number of CoTs queried from an expert demonstrator; for RLVR, $\ncomp$ is the number of model rollouts during training.}
    \label{tab:results}
\end{table}

\subsection{Organization}

\noindent In \Cref{sec:prelim} we introduce preliminaries, in \Cref{sec:distill} we introduce some prior results on CoT learning with distillation feedback and describe how adaptive queries can significantly reduce labeling costs, focusing on the setting of deterministic models. In \Cref{subsec:stoch} we extend these results to the setting where the learner operates with general models. In \Cref{sec:partial-coverage} we consider the RLVR setting where the learner has access to a reference model and show how autocurriculum reduces the \textit{computational cost} of learning good policies.

\medskip
\noindent This is the first of a two part series of papers on the role of autocurriculum in learning autoregressive models. In the sequel~\citep{rajaraman2026learningreasoncurriculumii}, we study a setting where the learner can \textit{design} its own curriculum of problems to facilitate learning.

\section{Preliminaries} \label{sec:prelim}

We introduce the formal setup in three parts: the autoregressive language modeling framework, the notion of accuracy with respect to an outcome verifier, and the SFT and RL settings.

\paragraph{Basic notation.} For nonnegative quantities $a,b$, $a \lesssim b$ (resp.\ $a \gtrsim b$) if there exists a universal constant $C>0$ such that $a \le C b$ (resp.\ $a \ge C b$) and $a \asymp b$ if both $a \lesssim b$ and $a \gtrsim b$ hold. We use standard Big-Oh notation: $a = O(b)$ means $|a| \le C b$ for a universal constant $C$, and $a=\Omega(b)$ denotes the reverse inequality; we write $a=\Theta(b)$ when both bounds hold. Finally, we use the $\|$ delimiter to separate core arguments from noteworthy hyperparameters, e.g., $\Alg (\cdots \| \cdots )$.

\paragraph{Language models.} Let $\Sigma$ denote a finite token space and $\Sigma^*$ denote the collection of (possibly infinite length) strings constructed from this vocabulary. We consider prompted language models, where $\mathcal{X} \subseteq \Sigma^*$ denotes an abstract prompt space, and $\rho \in \Delta_{\mathcal{X}}$ is a target distribution over prompts to learn under. Let $\pi : \Sigma^* \to \Delta_{\Sigma}$ denote a \textit{next-token predictor model}, which takes in a prefix of tokens and returns a next-token distribution. We consider learning in the \textit{parameter-sharing regime}: a single model $\pi$ is shared across every position in the sequence (as in standard autoregressive decoding). Generating autoregressively from $\pi$ induces a distribution over sequences, i.e., a chain-of-thought.

\begin{definition}[Chain-of-thought (CoT) distribution]
For a model $\pi$, the chain-of-thought distribution, $\pi_{1:T}$, is a conditional distribution over length-$T$ sequences: for a prompt $\bx \in \calX$, $\pi_{1:T} (\cdot|\bx)$ is the distribution over $(y_1,\cdots,y_T) \in \Sigma^T$ obtained by the iterative process: $\forall t \ge 1,\ y_t \sim \pi ( \cdot | \bx, y_1,\cdots,y_{t-1})$. $\pi_{1:T}$ is supported on the space of responses, $\calY = \Sigma^T$.
\end{definition}

\noindent In this abstraction of language model reasoning, we assume that the final token generated in the chain-of-thought corresponds to the ``answer'' of the model to a given prompt. Thus, we also introduce notation for the distribution over the final token generated by the model.

\begin{definition}[Outcome distribution]
Given a prompt $\bx \in \calX$ and a model $\pi$, let $\pi_T ( \cdot| \bx )$ denote the $\bx$-conditional distribution of $y_T$ for $(y_1,\cdots,y_T) \sim \pi_{1:T} (\cdot|\bx)$, supported on $\Sigma$.
\end{definition}

\paragraph{Accuracy and verification.} The learner's objective is to return a model that correctly predicts the final answer with high probability. To measure correctness, we assume access to an \textit{outcome verifier} $\calV$ (e.g., a unit test or answer checker) that determines whether a given answer is correct for a given prompt.

\begin{definition}[Outcome verifier] \label{def:outcome-verifier}
An outcome verifier is a function  $\calV : \calX \times \Sigma \to \{ 0,1 \}$ which takes a prompt $\bx \in \calX$ and a guess for the final answer $y \in \Sigma$, and returns $1$ if and only if $y$ is a correct answer for $\bx$.
\end{definition}

\noindent Given a model $\pi$, its accuracy with respect to the outcome verifier $\calV$ is defined as the average accuracy across prompts in predicting the final answer correctly,
\begin{align}
    \acc_{\rho} ( \pi_{1:T} \| \calV ) \triangleq \bbE_{\bx \sim \rho} \big[ \acc_{\bx} (\pi_{1:T} \| \calV ) \big], \text{ where } \acc_{\bx} ( \pi_{1:T} \| \calV ) = \bbE_{\by \sim \pi_{1:T} (\cdot|\bx)} \big[ \calV ( \bx, y_T ) \big].
\end{align}
When $\calV$ and $T$ are clear from context, we abbreviate $\acc_{\rho} (\pi_{1:T} \| \calV)$ as $\acc_{\rho} (\pi)$ (and similarly $\acc_{\bx} ( \pi_{1:T} \| \calV )$ as $\acc_{\bx} (\pi)$).

\medskip
\noindent Throughout the paper, we consider a realizable setting where the learner has access to a model class $\Pi$ containing a model $\pi^\star$ which \textit{achieves perfect accuracy}.

\begin{assumption}[Optimal model is realizable] \label{assump:teacher}
There exists a model $\pi^\star \in \Pi$ such that,
\begin{equation*}
    \acc_\rho ( \pi^\star ) = 1.
\end{equation*}
\end{assumption}

\paragraph{Learning settings.} We consider two learning settings. In the SFT setting, the learner can query a teacher for reasoning traces on chosen prompts. In the RLVR setting, the learner instead has access to a pre-trained reference model that can generate candidate traces, and must improve it using verifier feedback.

\begin{definition}[Learning settings] \label{def:base-learner}
Consider a dataset of prompts, $D = \{ \bx_i \}_{i=1}^n$, where $\bx_i \sim \rho$. A learning algorithm $\Alg_\calQ (D \| \varepsilon, \delta, T)$ returns a model \smash{$\pihat$} such that as long as \smash{$n \ge \nprompt (\varepsilon,\delta,T)$}, w.p. at least $1-\delta$,
\begin{equation*}
    \acc_\rho \big( \pihat \big) \ge 1 - \varepsilon.
\end{equation*}
$\nprompt (\varepsilon,\delta,T)$ is referred to as the sample complexity of the learning algorithm. The subscript $\calQ \in \{ \texttt{SFT}, \texttt{RL} \}$ indicates the learning setting and how costs are measured.
\begin{itemize}[leftmargin=*]
    \item \textit{SFT setting.} For $\bx \in \calX$, a CoT oracle $\CoT (\bx)$ returns $\by \sim \pi_{1:T}^\star (\cdot|\bx)$. In the SFT setting, we assume that the learner has query access to $\CoT(\cdot)$. The query complexity, denoted $\nCoT (\varepsilon,\delta,T)$, measures the number of queries to $\CoT(\cdot)$ made by $\Alg_{\texttt{SFT}} (\cdot \| \varepsilon,\delta,T)$.
    \item \textit{RL setting.} A reference model is a conditional distribution \smash{$\piref : \calX \to \Delta_{\Sigma^T}$}. In the RL setting, we assume query access to it: for $\bx \in \calX$, we obtain a response \smash{$\Sigma^T \ni \by \sim \piref (\cdot|\bx)$} as well as its probability $\piref (\by|\bx)$. The computational cost of the learner, denoted $\ncomp (\varepsilon,\delta,T)$, measures the total number of length-$T$ responses generated from $\piref$ or any other model in $\Pi$ during the execution of $\Alg_{\texttt{RL}} (\cdot \| \varepsilon,\delta,T)$.
\end{itemize}
In both settings, the learner also has query access to the outcome verifier $\calV$.
\end{definition}

\section{SFT: Fine-Tuning with Teacher Supervision} \label{sec:distill}

We begin by reviewing prior results on SFT without curriculum (\Cref{sec:sft}), then motivate and formalize the autocurriculum problem (\Cref{sec:autocurriculum-sft}). Our main result, an exponential improvement in the number of CoT demonstrations for deterministic models, is in \Cref{subsec:det}; the extension to general models is in \Cref{subsec:stoch}.\loose

\subsection{Prior Work: SFT without Curriculum} \label{sec:sft}

\noindent The simplest approach to SFT is to collect CoT demonstrations from the teacher on every training prompt, then fit the model via next-token prediction on the resulting dataset. \citet{joshi2025theory} and \citet{foster2024behavior} analyze this approach and show that it is statistically efficient: treating $n$ reasoning traces as $nT$ next-token examples and minimizing empirical risk yields a model with high accuracy. The following proposition, which is a corollary of their results, summarizes the baseline sample complexity.

\begin{proposition}[Corollary of \citet{foster2024behavior,joshi2025theory}] \label{theorem:prior}
Let next-token prediction, $\NTP$, denote the CoT-supervised learning algorithm which takes a dataset $D = \{ (\bx_i,\by_i) \}_{i=1}^n$ of CoTs with $\bx_i \sim \rho$ and $\by_i \sim \pi^\star_{1:T} (\cdot|\bx_i)$, and returning the model,
\begin{equation*}
    \pi^{\NTP} \in \argmin_{\pi \in \Pi} \frac{1}{nT} \sum_{i=1}^n \sum_{t=1}^T \ell \big( \pi (\cdot | \bx_i,\by_{i,1:t-1}), y_{i,t} \big)
\end{equation*}
Under \Cref{assump:teacher}, with $\ell$ as the log-loss ($\ell (\pi,a) \equiv \log (1/\pi(a))$), $\NTP$ has sample and query complexity,
\begin{equation*}
    \nprompt (\varepsilon,\delta,T) = \nCoT (\varepsilon,\delta,T)  \lesssim \frac{\log (|\Pi|/\delta)\log(1/\varepsilon)}{\varepsilon^2}
\end{equation*}
Furthermore, when $\Pi$ is a (potentially unbounded) family of deterministic models, with $\ell$ as the $0$-$1$ loss, ($\ell (\pi,a) \equiv \bbE_{a' \sim \pi} [\bbI (a' \ne a)]$), under \Cref{assump:teacher}, the sample and query complexity of $\NTP$ is upper bounded by,
\begin{equation} \label{eq:prior}
    \nprompt (\varepsilon,\delta,T) = \nCoT (\varepsilon,\delta,T) \lesssim \frac{d \cdot \log(d T |\Sigma| ) \log(1/\varepsilon) + \log(1/\delta)}{\varepsilon}
\end{equation}
where $d = \Ndim (\Pi) \le \log_2 (|\Pi|)$ is the Natarajan dimension of $\Pi$ \citep{natarajan1989learning}.
\end{proposition}

\subsection{Reducing the Cost of Supervision: Autocurriculum} \label{sec:autocurriculum-sft}

\noindent The baseline approach in \Cref{theorem:prior} collects CoT demonstrations on \textit{every} training prompt. But this is wasteful: many prompts may already be easy for the current model, and collecting expensive reasoning traces on them provides little benefit. A natural idea is to let the learner \textit{choose} which prompts to query for CoTs, focusing supervision on the prompts where it is most needed. To formalize this, we assume that in addition to the CoT oracle, the learner has access to the outcome verifier $\calV$, which can cheaply evaluate whether the model's current answer is correct.

\begin{problem}[Autocurriculum for SFT] \label{problem:best-of-both-worlds}
The learner receives a dataset of $n$ prompts $D = \{ \bx_i \}_{i=1}^n$ where \smash{$\bx_i \overset{\text{i.i.d.}}{\sim} \rho$} and has query access to the outcome verifier $\calV$ (\Cref{def:outcome-verifier}), and to CoT supervision oracle $\CoT : \calX \to \Delta_{\Sigma^T}$ (\Cref{def:base-learner}). The objective is to return a model $\pihat$ such that $\acc_{\rho} ( \pihat ) \ge 1-\varepsilon$.
\end{problem}

\noindent Note that this setting is \textit{not} a special case of active learning: the verifier provides a cheap source of feedback (is the current model correct on this prompt?) that is distinct from the expensive CoT supervision. This is what enables the exponential savings we establish below.

\subsection{Main Result: Exponential Improvement in CoT Supervision via Autocurriculum} \label{subsec:det}

We first specialize to deterministic models, i.e., each $\pi \in \Pi$ maps each token prefix to a single token. While language models trained in practice are usually not deterministic, we consider this setting as a warmup for the general case discussed in \Cref{subsec:stoch}. The deterministic setting is also closely tied to problems like semiautomaton learning~\citep{giapitzakis2025statistical} which have received attention recently. We also assume that each prompt has a unique correct answer (\Cref{assump:sparse} below), whereby accuracy reduces to the probability of matching the teacher's answer\footnote{We overload notation so that for any $\pi \in \Pi$, $\pi_T (\bx)$ denotes the token $\pi_T (\cdot|\bx)$ is supported on.}
\begin{equation} \label{eq:acc-rhoT}
    \acc_{\rho} \big( \pi \big) = \bbE_{\bx \sim \rho} \big[ \bbI \big( \pi_T (\bx) = \pi^\star_T (\bx) \big) \big].
\end{equation}

\begin{assumption}[Unique answers (sparse verifier)] \label{assump:sparse}
The verifier $\calV$ is supported on a unique correct answer for each prompt $\bx \in \calX$: $\{ y \in \Sigma : \calV ( \bx, y) = 1 \}$ is a singleton set. Under \Cref{assump:teacher}, this implies that $\pi^\star_T (\cdot|\bx)$ must be supported only on the corresponding token in this set, almost surely over $\bx \sim \rho$.
\end{assumption}

\paragraph{Algorithm idea.} Our algorithm, $\Genboost$ (\Cref{alg:CoTboost}), is inspired by classical boosting~\citep{Freund,schapire2013boosting}: it iteratively trains an ensemble of models, where each new model is trained using next-token prediction ($\NTP$) on prompts that the current ensemble gets wrong. Crucially, the verifier determines which prompts to focus on, while CoT demonstrations are only collected for those prompts. Since each new model fixes a constant fraction of the remaining errors, $k = \calO(\log(1/\varepsilon))$ rounds suffice to reach accuracy $1-\varepsilon$, granting an \textit{exponential} improvement in the number of CoT demonstrations required over the non-adaptive baseline. The resulting algorithm is improper, inducing an ``outcome-level'' model $\fhat : \calX \to \Delta_\Sigma$ obtained by aggregating the answers of the models in the ensemble.

\begingroup
\setlength{\algomargin}{\dimexpr\algomargin+\algonuminset\relax}
\normalem
\begin{algorithm}[htp]
\caption{\texorpdfstring{$\Genboost \big(\Dprompt{} \big\| \Alg_\calQ, \varepsilon,\delta,T \big)$}{AutoTune}}
\label{alg:CoTboost}
\noindent\hspace*{-\algonuminset}
\begin{minipage}{\dimexpr\linewidth+\algonuminset\relax}

\LinesNumbered
\setcounter{AlgoLine}{0}
\DontPrintSemicolon

{\color{gray}\# Supervised fine-tuning of deterministic policies with autocurriculum.}\;

\nl \KwIn{Class of models $\Pi$; target accuracy $1-\varepsilon$ and failure prob. $\delta$; reasoning steps $T$;\;

\hspace{3.2em}Prompt dataset $\Dprompt{} = \{\bx_i\}_{i=1}^n$ where $\bx_i \sim \rho$;\;

\hspace{3.2em}Outcome verifier $\calV:\calX\times\Sigma\to\{0,1\}$;\;

\hspace{3.2em}Base learning algorithm $\Alg_\calQ (\cdot \| \varepsilon',\delta',T)$.}

\BlankLine
\textbf{Initialize:} $\err_\star \gets \frac{1}{4}$ and $k \gets \left\lceil C\log(1/\varepsilon)/\err_\star^2 \right\rceil$ for a large absolute constant $C>0$.\;

Split $\Dprompt{}$ into $k$ equal parts $\big\{ \Dprompt{j} : 0 \le j \le k-1 \big\}$\;

$\Pi_0 \gets \emptyset$\;

\For{\textbf{\textup{phase}} $j \gets 0$ \KwTo $k-1$}{ \label{alg:CoTboost-4}
  $\Dout{j} \gets \subsample \big(\Dprompt{j} \big\| \Pi_j, k\big)$\label{alg:CoTboost-sample}
  \cmt*[r]{$\subsample(\cdot)$ (\Cref{alg:CoTboost-subsample}) induces target learning}

  $\pihat^j_T \gets \Alg_\calQ \big(\Dout{j} \big\| \err_\star, \frac{\delta}{k}, T \big)$ \label{alg:CoTboost-weaklearner}
  \cmttwo*[r]{\smash{\raisebox{0.3em}{distribution via rejection sampling from $\rho$.}}}

  $\Pi_{j+1} \gets \Pi_j \cup \big\{ \pihat^j \big\}$\;
}

\Return $\maj( \{ \pihat_T : \pihat \in \Pi_k \} )$\;
\end{minipage}
\end{algorithm}
\ULforem
\endgroup

\begin{theorem}[Exponential improvement via autocurriculum for SFT] \label{theorem:CoTboost}
Consider any $\delta \in (0,1/2)$ and $\varepsilon \in (0,1)$. Suppose $\Pi$ is composed of deterministic models, the verifier is sparse (\Cref{assump:sparse}) and that the optimal model is realizable (\Cref{assump:teacher}). Let $\Genboost$ (\Cref{alg:CoTboost}) be instantiated with the base learner $\Alg_{\SFT} (\cdot \| \ \varepsilon,\delta,T)$ from \Cref{theorem:prior}. If the size of the prompt dataset input to $\Genboost$ satisfies,
\begin{equation} \label{eq:samplecomplexity}
    n \ge \frac{d}{\varepsilon} \cdot \polylog ( \varepsilon^{-1}, \delta^{-1}, T, |\Sigma| ).
\end{equation}
where $d = \Ndim (\Pi)$. Then, the number of times $\Genboost$ queries the verifier is upper bounded by $n$, and the number of queries to the CoT oracle $\CoT (\cdot)$ is upper bounded by,
\begin{equation*}
    \nCoT \le d \cdot \polylog(\varepsilon^{-1},\delta^{-1},T,|\Sigma|)
\end{equation*}
and the resulting outcome-level model $\fhat$ satisfies with probability at least $1-\delta$, $\acc_{\rho} \big( \fhat \big) \ge 1- \varepsilon$.
\end{theorem}
\begin{proof}
The proof is deferred to \Cref{sec:CoTboost}.
\end{proof}

\paragraph{Comparisons to baseline algorithms.}

\noindent \textit{CoT feedback without autocurriculum.} As discussed in \Cref{sec:sft}, \citet{joshi2025theory} show that given an i.i.d. dataset of $\calOtilde \big( \frac{d}{\varepsilon} \big)$ prompts labeled by CoTs from $\pi^\star$,\footnote{$d = \Ndim (\Pi) \le \log_2 (|\Pi|)$ denotes the Natarajan dimension of $\Pi$.} next-token prediction ($\NTP$) can be used to train a model to accuracy $1-\varepsilon$, and that this is optimal in the worst case. In \Cref{theorem:CoTboost} we show that it is possible to achieve accuracy $1-\varepsilon$ querying only $\calOtilde (d)$ CoTs from $\pi^\star$, using autocurriculum: an exponential improvement in query complexity.

\begingroup
\setlength{\algomargin}{\dimexpr\algomargin+\algonuminset\relax}
\normalem
\begin{algorithm}[t]
\caption{\texorpdfstring{$\subsample (D \| \Pi_j, k )$}{Sample}}
\label{alg:CoTboost-subsample}
\noindent\hspace*{-\algonuminset}
\begin{minipage}{\dimexpr\linewidth+\algonuminset\relax}
\LinesNumbered
\setcounter{AlgoLine}{0}
\DontPrintSemicolon
{\color{gray}\# Subsampling prompts via rejection sampling.}\;

\nl \KwIn{Dataset of prompts, $D = \{ \bx_i \}_{i=1}^n$\;

\hspace{2.9em} Set of $j$ models $\Pi_j$;\;

\hspace{2.9em} Outcome verifier $\calV : \calX \times \Sigma \to \{ 0,1 \}$;\;

\hspace{2.9em} Base learning algorithm $\Alg_\calQ (\cdot \| \varepsilon, \delta,T)$ with sample complexity $\nprompt (\varepsilon,\delta,T)$.
}

\BlankLine

\textbf{Initialize:} $\err_\star \gets \frac{1}{4}$ and $\Dout{} \gets \emptyset$\;

\For{prompts $\bx \in D$}{

\If(\hfill \cmt*[f]{If $\Dout{}$ is sufficiently large that $\Alg_\calQ$ can train a}){$|\Dout{}| \ge \nprompt \big( \err_\star,  \frac{\delta}{k} , T \big)$}{ \label{alg:CoTboost-subsample-6}
\textbf{Return:} $\Dout{}$ \cmttwo*[r]{model with constant accuracy, terminate loop.}
}

Draw $\eta \sim \unif ([0,1])$.\;

\If(\hfill \cmt*[f]{$w_j$ and $\| w_j \|_\infty$ are defined in \cref{eq:alpha}. \label{alg:CoTboost-subsample-7}}){$\eta \le w_j (\bx) / \| w_j \|_\infty$}{

$\Dout{} \gets \Dout{} \cup \{ \bx \}$\;
}
}

\textbf{Return:} $\Dout{}$\;

\vspace{-0.25em}
\nonl \hrulefill
\vspace{0.25em}

\nonl For $0 \le r \le j < k$, the weight $w_j$ is defined as $w_j (\bx) = \alpha^{j,k}_{\rank_j (\bx)}$ and $\| w_j \|_\infty = \max\limits_{0 \le r \le j} \alpha^{j,k}_r$, where,
\vspace{-0.75em}
\begin{align}
    \alpha^{j,k}_r &= \beta_r^{j+1,k} - \beta_{r+1}^{j+1,k}, \text{ where, } \beta^{j,k}_r = \begin{cases}
    \bbI (r \le k/2) &\text{if } j = k \\
    \err_\star \cdot \beta_r^{j+1,k} + \left( 1 - \err_\star \right) \cdot \beta_{r+1}^{j+1,k} \quad &\text{if } j < k
\end{cases}
\label{eq:alpha}
\end{align}
$\rank_j (\bx)$ counts the number of models in $\Pi_j$ which guess the label on $\bx$ correctly: for $j \ge 0$,
\begin{equation} \label{eq:rank}
    \rank_j (\bx) = \sum\nolimits_{\pi \in \Pi_j} \calV (\bx, \pi_T (\bx)) \in [0,j].
\end{equation}
\vspace{-1em}
\end{minipage}
\end{algorithm}
\ULforem
\endgroup

\noindent \textit{End-to-end feedback.} \citet{joshi2025theory} also consider an ``end-to-end'' learning setting where the learner is given a dataset of $n$ prompts labeled only with the final answer, $\{ (\bx_i,y_i) \}_{i=1}^n$ with $\bx_i \sim \rho$ and $y_i \sim \pi^\star_T (\cdot|\bx_i)$, and aims to learn a model with high accuracy. Here the authors show strong statistical and computational lower bounds for learning with this feedback: the worst-case sample complexity degrades to $\widetilde{\Omega} \big( \frac{dT}{\varepsilon} \big)$, compared to the setting where the learner receives prompts labeled with full CoTs (cf. \cref{eq:prior}). While this can be attributed to the fact that CoTs are more informative, there are also computational barriers: learning a model that is much better than random guessing, from end-to-end feedback, is shown to be computationally intractable.

\smallskip
\noindent In contrast, \Cref{theorem:CoTboost} shows that even assuming a weaker form of supervision than end-to-end feedback (namely, access to an outcome verifier), a small amount of carefully selected CoT feedback is sufficient to restore computational tractability (\Cref{alg:CoTboost} is efficient with respect to an ERM oracle for $\Pi$ under the $0$-$1$ loss). Furthermore, from a statistical point of view, the prompt dataset size required by $\Genboost$ to achieve accuracy $1-\varepsilon$ (cf. \Cref{eq:samplecomplexity}) returns to $\calOtilde \big( \frac{d}{\varepsilon} \big)$. 

\medskip
\noindent Taken together, these results show that autocurriculum enables best-of-both-worlds guarantees: the query complexity of autocurriculum improves exponentially over learning from CoTs alone, while simultaneously avoiding the statistical and computational barriers that are inherent to learning from end-to-end feedback.

\begin{remark}[Comparison to active learning]
In the active learning literature, $\log(1/\varepsilon)$-style label complexity bounds for classification are achievable under distributional assumptions such as bounded star number~\citep{hanneke2015minimax} or bounded disagreement coefficient~\citep{hanneke2007bound} on the hypothesis class. In contrast, the CoT query complexity for $\Genboost$ in \Cref{theorem:CoTboost} has polylogarithmic dependence in $1/\varepsilon$ without any such assumptions on $\Pi$ or the induced end-to-end class $\Pi_T = \{ \pi_T : \pi \in \Pi \}$. This highlights that \Cref{problem:best-of-both-worlds} is not a special case of active learning: the learner can extract information about $\pi^\star$ by querying the verifier, which is sufficiently informative to circumvent active learning lower bounds.
\end{remark}

\begingroup
\normalem
\begin{figure}[t!]
\centering
\subfigure[$\pihat_1$ trained to accuracy at least $\frac{3}{4}$.]{
\begin{tikzpicture}
  \tkzDefPoint(0,0){A}  
  \tkzDefPoint(1.25,0.25){B}
  \begin{scope}
  \tkzFillCircle[color=gray!20](B,A)
  \end{scope}
  \tkzDrawCircle(B,A)
  \node at (2.25,-1) {$\pihat_1$};
  \node at (1.25,0.25) {$\ge \frac{3}{4}$};
  \node[below left] at (2.9,2) {$\calX$};
  \path[use as bounding box, draw] (-1.65,-2.5) rectangle (2.9,2);
\end{tikzpicture}
}\hfill
\subfigure[Model $\pihat_2$ is trained subsequently to accuracy $\frac{3}{4}$.]{
\begin{tikzpicture}
  \tkzDefPoint(0,0.25){A}
  \tkzDefPoint(1.25,0.25){B}
  \begin{scope}
    \tkzClipCircle(A,B) \tkzClipCircle(B,A)
    \tkzFillCircle[color=gray!20](A,B)
  \end{scope}
  \tkzDrawCircle(A,B)
  \tkzDrawCircle(B,A)
  \node at (-1,-1) {$\pihat_2$};
  \node at (2.25,-1) {$\pihat_1$};
  \path[use as bounding box, draw] (-1.65,-2.5) rectangle (2.9,2);
\end{tikzpicture}
}\hfill
\subfigure[$\pihat_3$ fixes labels on some incorrectly labeled prompts.]{
\begin{tikzpicture}
  \tkzDefPoint(0,0.25){A}
  \tkzDefPoint(1.25,0.25){B}
  \tkzDefPoint(0.625,-0.84){C}
  \begin{scope}
    \tkzClipCircle(A,B) \tkzClipCircle(B,A)
    \tkzFillCircle[color=gray!20](A,B)
  \end{scope}
  \begin{scope}
  \tkzClipCircle(B,C) \tkzClipCircle(B,C)
    \tkzFillCircle[color=gray!20](C,B)
  \end{scope}
  \begin{scope}
    \tkzClipCircle(C,A) \tkzClipCircle(C,A)
    \tkzFillCircle[color=gray!20](A,C)
    \end{scope}
  \tkzDrawCircle(A,B)
  \tkzDrawCircle(B,A)
  \tkzDrawCircle(C,A)
  \node at (-0.5,-2) {$\pihat_3$};
  \node at (-1,-1) {$\pihat_2$};
  \node at (2.25,-1) {$\pihat_1$};
  \node at (1.75,0.5) {{\color{red}\large$\times$}};
  \node at (-0.15,-0.5) {{\color{green!50!black}\large$\checkmark$}};
  \path[use as bounding box, draw] (-1.65,-2.5) rectangle (2.9,2);
\end{tikzpicture}
}
\caption{An illustration of how models trained on the appropriate prompt distributions can correct errors of prior ones. The region marked in gray captures the region correctly labeled by the plurality of the ensemble of models. Comparing $(a)$ to $(c)$, the model $\pihat_3$ corrects some errors made by $\pihat_1$ (green checked region), but also introduces new errors (red crossed region). When trained under the appropriate autocurriculum, the accuracy of the ensemble improves geometrically toward $1$.}
\label{fig:recursive-halving}
\end{figure}
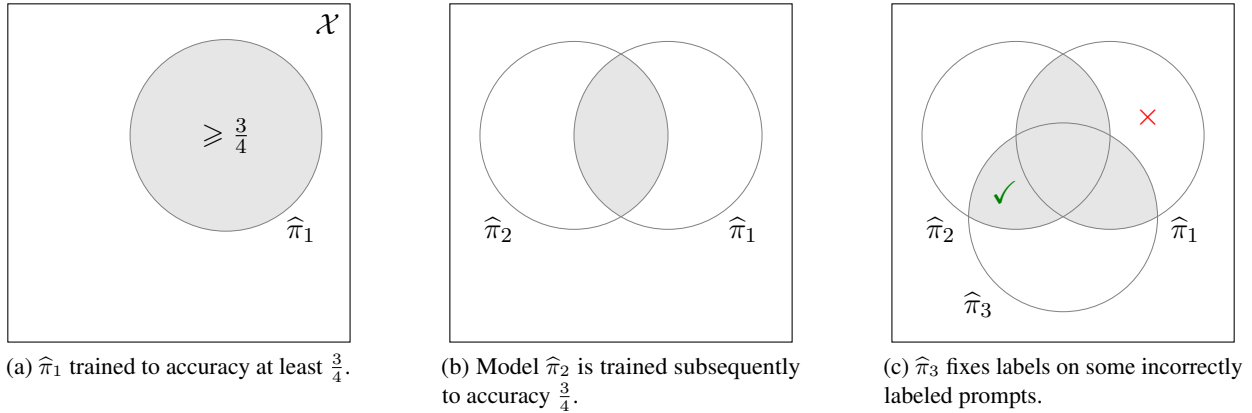
\ULforem
\endgroup

\begin{remark}[The role of outcome-level accuracy]
The bound on the number of CoTs queried by \Cref{alg:CoTboost} in \Cref{theorem:CoTboost} (or any algorithm, for that matter) is only achievable when the accuracy \smash{$\acc_{\rho} ( \pihat )$} of a model $\pihat$ is measured as in \Cref{eq:acc-rhoT} as the correctness in predicting the final token $\pi^\star_T (\bx)$ for $\bx \sim \rho$. In particular, if the accuracy of $\pihat$ is instead measured by its correctness on the full CoT sequence, $\bbE_{\bx \sim \rho} \big[ \bbE_{\by \sim \pihat_{1:T} (\cdot|\bx)} \big[ \bbI \big( \by = \pi^\star_{1:T} (\bx) \big) \big] \big]$, the bound degrades to \smash{$\widetilde{\Omega} \big( \frac{d}{\varepsilon} \big)$} CoT queries. This follows from the fact that the outcome verifier cannot provide any supervision for the intermediate tokens in the teacher's CoT.\loose
\end{remark}

\paragraph{Proof sketch for \Cref{theorem:CoTboost}.} The high-level idea is illustrated in \Cref{fig:recursive-halving}. In each of the $k = \calO(\log(1/\varepsilon))$ phases, a new model is trained to constant accuracy on a reweighted distribution over prompts that upweights the regions where the current ensemble errs. The $\subsample(\cdot)$ subroutine (\Cref{alg:CoTboost-subsample}) implements this reweighting using the verifier. Each new model is trained to be a \textit{weak learner}, fixing a constant fraction of remaining errors. Thereby, the ensemble's accuracy approaches $1$ at a geometric rate. A subtle point is that weak models trained to constant accuracy may introduce new errors when added to the ensemble. The key argument in the proof is that the reweighting ensures each new model fixes mistakes on a larger measure of prompts than it introduces errors on.

\noindent \textit{Distributions induced by $\subsample (\cdot)$.} The distribution over prompts induced by $\subsample(\cdot)$ in iteration $j$, $\rho_j^\star$, can be written down in terms of a reweighting function $w_j$ applied to the original prompt distribution $\rho$
\begin{align} \label{eq:rhojstar}
    \rho_j^\star (\cdot) &\propto \rho (\cdot) w_j (\cdot), \text{ where } w_j (\bx) = \alpha^{j,k}_{\rank_j(\bx)} \text{ and } \| w_j \|_\infty = \max_{0 \le r \le j} \alpha^{j,k}_r
\end{align}
where $\alpha^{j,k}_r$ and $\rank_j$ are defined in \cref{eq:alpha} and \cref{eq:rank} respectively. The choice of this weight function is inspired from the boosting-by-filtering approach~\citep{Freund}. \smash{$\alpha^{j,k}_{\rank_j(\bx)}$} has a natural interpretation: it captures the probability that more than half of the models in the final ensemble predict the correct label on $\bx$, given that $\rank_j (\bx)$ among the first $j$ predict the correct label, and assuming that every future model is independently correct on $\bx$ with probability $1-\err_\star = \frac{3}{4}$. 

\subsection{Extension to General Models} \label{subsec:stoch}

We now extend our results to the case where $\Pi$ contains stochastic models, which more accurately models the language modeling setting. We show that autocurriculum yields significant improvements in this setting as well, provided that the correctness guarantee is relaxed from \textit{high accuracy} to \textit{moderate accuracy on most prompts}. This can be interpreted as a bound on the model's \textit{pass@k rate}~\citep{chen2021evaluating}: the probability of a positively rewarded attempt among $k$ independently generated answers. Furthermore, these guarantees can be \textit{sharpened} to high-accuracy ones at inference time when prompts have unique final correct answers.

\begin{theorem}[Autocurriculum for SFT with general models] \label{theorem:CoTboost-stochastic}
Consider any $\delta \in (0,1/2)$ and $\varepsilon \in (0,1)$ and assume that the optimal model is realizable (\Cref{assump:teacher}). Let $\Genboost$ (\Cref{alg:CoTboost}) use the base algorithm $\Alg_{\SFT}$ as the learner from \Cref{theorem:prior}. Suppose the size of the prompt dataset satisfies, $n \ge \frac{\log (|\Pi|)}{\varepsilon} \cdot \polylog ( \varepsilon^{-1}, \delta^{-1}, T )$. Then, $\Genboost$ queries the oracle $\CoT (\cdot)$ at most,
\begin{equation*}
    \nCoT \le \log (|\Pi|) \cdot \polylog(\varepsilon^{-1},\delta^{-1},T)
\end{equation*}
times to return an outcome-level model $\fhat$ such that with probability $1-\delta$, $\Pr_{\bx \sim \rho} \big( \acc_\bx ( \fhat ) \ge \frac{3}{5} \big) \ge 1-\varepsilon$. Finally, $\Genboost$ is implementable with $\calO(\log(1/\varepsilon))$ calls to a log-loss ERM oracle for $\Pi$.
\end{theorem}
\begin{proof}
The proof is deferred to \Cref{app:CoTboost-stochastic}.
\end{proof}

\noindent For the case of deterministic model classes considered earlier, \Cref{alg:CoTboost} constructs an ensemble of models such that at the end of training, on a $1-\varepsilon$ mass of prompts, over half of the models guess the label correctly. The proof follows a similar strategy to the deterministic case, except that the $0$-$1$ loss is replaced by
\begin{equation} \label{eq:loss-stochastic}
    \ell_\bx (\pi) = \bbI \Big( \acc_\bx (\pi) \ge \frac{4}{5} \Big)
\end{equation}
We defer a detailed discussion to the appendix, where we address the technical complication that \cref{eq:loss-stochastic} cannot be computed exactly when the model $\pi$ is stochastic. This changes how the weak learning distributions are defined.

\begin{remark}[Sharpening by consensus vote]
Assuming that the answers to prompts are unique (\Cref{assump:sparse}), it is possible to improve the guarantee of \Cref{theorem:CoTboost-stochastic} to a high accuracy one by consensus vote. The resulting algorithm, $\Genboost_{\texttt{cons}}$ (inducing the outcome-level model $\fhatcons$), is defined as follows: for $\bx \in \calX$, $\fhatcons(\cdot|\bx)$ samples an answer by computing $\maj \big( \big\{ y^1,\cdots,y^N \big\} \big)$ for $N$ independent answers, \smash{$y^i \sim \fhat(\cdot|\bx)$}. For suitably large $N = \Omega (\log(1/\varepsilon))$, the outcome-level model returned by $\Genboost_{\texttt{cons}}$ satisfies the guarantee, \smash{$\acc_{\rho} (\fhatcons) \ge 1 - \calO(\varepsilon)$}.
\end{remark}

\begingroup
\setlength{\algomargin}{\dimexpr\algomargin+\algonuminset\relax}
\normalem
\begin{algorithm}[t]
\caption{\texorpdfstring{$\Genbooststoch \big( \Dprompt{} \big\| \Alg_{\SFT}, \varepsilon,\delta,T \big)$}{AutoLearnStoch}}
\label{alg:CoTboost-stochastic}
\noindent\hspace*{-\algonuminset}
\begin{minipage}{\dimexpr\linewidth+\algonuminset\relax}

\LinesNumbered
\setcounter{AlgoLine}{0}
\DontPrintSemicolon

{\color{gray}\# Supervised fine-tuning of stochastic policies with autocurriculum.}\;

\nl \KwIn{Class of models $\Pi$; target accuracy $1-\varepsilon$ and failure prob. $\delta$; reasoning steps $T$;\;

\hspace{3.1em}Prompt dataset $\Dprompt{} = \{\bx_i\}_{i=1}^n$ where $\bx_i \sim \rho$;\;

\hspace{3.1em}Outcome verifier $\calV:\calX\times\Sigma\to\{0,1\}$;\;

\hspace{3.1em}Base CoT-supervised learning algorithm $\Alg_{\SFT} (\cdot \| \varepsilon',\delta', T)$.}

Let $\widetilde{\Genboost} (\Dprompt{} \| \Alg_{\SFT}, \varepsilon,\delta,T)$ be a variant of $\Genboost$ (\Cref{alg:CoTboost}) with $\err_\star \gets \frac{1}{400}$ and where \Cref{alg:CoTboost-sample} is replaced by:
\begin{equation*}
    \Dout{j} \gets \widetilde{\subsample} \big( \Dprompt{j} \big\| \Pi_j , k \big),
\end{equation*}

\nonl $\widetilde{\subsample} ( D \| \Pi_j, k)$ is a variant of $\subsample$ (\Cref{alg:CoTboost-subsample}) with $\err_\star \gets \frac{1}{10}$ and weight function $\ww_j$ instead of $w_j$ on \Cref{alg:CoTboost-subsample-7} of \Cref{alg:CoTboost-subsample}. \cmt*[r]{$\ww_j$ is defined in \Cref{eq:montecarloacc} to \Cref{eq:trhojstar}}

Let \smash{$\pihat^0,\cdots,\pihat^{k-1}$} be the models trained within \smash{$\widetilde{\Genboost} (\Dprompt{} \| \Alg_{\SFT}, \varepsilon, \delta, T)$}\;

\textbf{Return:} Uniform mixture of outcome-level models, $\frac{1}{k} \sum_{j=0}^{k-1} \pihat^j_T$.\;
\end{minipage}
\end{algorithm}
\ULforem
\endgroup

\section{RL: Improving a Reference Model with Verifier Guidance} \label{sec:partial-coverage}

In the previous section, the learner had access to an optimal teacher. We now turn to the more practical RLVR setting~\citep{lambert2024tulu}, where the learner starts with a pre-trained reference model and must improve it using only verifier feedback. Since generation costs dominate in practice~\citep{zhou2025aprilactivepartialrollouts}, we measure computational cost by the total number of reasoning traces generated during training. We first formalize the problem and introduce a natural coverage assumption on the reference model, then establish a baseline showing that learning is tractable via rejection sampling, and finally show that autocurriculum decouples the dependence on coverage from accuracy.

\begin{figure}
    \centering
    \includegraphics[width=0.85\linewidth]{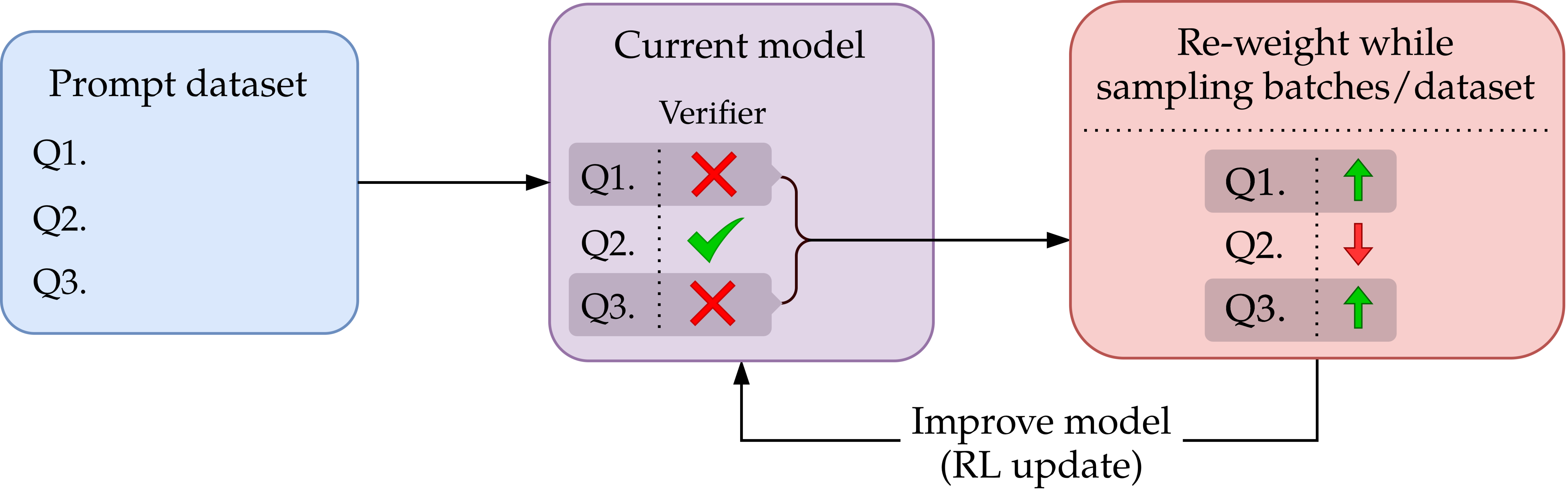}
    \caption{An example of autocurriculum for RL: the learner chooses which prompts to add to the training batch/dataset, based on accuracy. In each iteration, the learner's model is updated on these prompts using an RL update to improve its accuracy.}
    \label{fig:autocurriculum_RL}
\end{figure}

\begin{problem}[Fine-tuning a reference model]
\label{problem:autoRL}
The learner receives a dataset of $n$ prompts $D = \{ \bx_i \}_{i=1}^n$ where \smash{$\bx_i \overset{\text{i.i.d.}}{\sim} \rho$} and has query access to the outcome verifier $\calV$, and to a reference model $\piref : \calX \to \Delta_{\Sigma^T}$. The objective is to return a model $\pihat$ such that $\acc_{\rho} ( \pihat ) \ge 1-\varepsilon$.

\smallskip
\noindent The computational cost of the learner is measured as the total number of CoTs generated over the course of training: this includes those from $\piref$, as well as those from the learner's own models.
\end{problem}

\noindent In the natural regimes of RL finetuning in practice, one starts from a reference model that is trained to generate correct reasoning traces with non-negligible probability. We formalize this through \textit{sequence-level coverage}~\citep{rashidinejad2021bridging,foster2025good}, which requires that the reference model places at least $\Cseq^{-1}$ probability on the correct CoT for each prompt. The parameter $\Cseq$ quantifies how well the reference model ``covers'' the target behavior; smaller $\Cseq$ means better coverage.

\begin{definition}[Sequence-level coverage~\citep{jiang2025offline}] \label{def:seq-cov}
Assume the model class $\Pi$ is deterministic. A reference model $\piref$ is said to satisfy sequence-level coverage with parameter $\Cseq$ if,
\begin{equation} \label{eq:pcov}
    \piref (\by_\bx^\star | \bx) \ge \Cseq^{-1},
\end{equation}
where, $\by^\star_\bx = \pi_{1:T}^\star (\bx)$ denotes the CoT generated by $\pi^\star$ on the prompt $\bx$.
\end{definition}

\subsection{Baseline: Learning via Rejection Sampling}

As a baseline, we consider the algorithm $\Vlearn$ (\Cref{alg:Vlearn}), which follows a natural approach analogous to filtered SFT: for each prompt, generate multiple candidate CoTs from the reference model, keep only those that produce a correct answer (as judged by the verifier), and train a model on the surviving traces. Under the coverage assumption, the reference model is guaranteed to produce at least one correct trace per prompt with high probability after $\calOtilde(\Cseq)$ samples. 

\begin{proposition}[Learning from a reference model with coverage] \label{theorem:seq-coverage}
Consider any target error $\varepsilon \in (0,1)$ and failure probability $\delta \in (0,1)$. Suppose $\Pi$ is deterministic, the reference model $\piref$ satisfies sequence-level coverage with parameter $\Cseq$ (\Cref{def:seq-cov}), answers are unique (\Cref{assump:sparse}) and that the optimal model is realizable (\Cref{assump:teacher}). Consider the model \smash{$\pihat$} returned by $\Vlearn$ (\Cref{alg:Vlearn}). There exists an absolute constant $C_1 > 0$ such that if,
\begin{equation*}
    n \ge C_1 \frac{d \log (T |\Sigma| \Cseq) \log(1/\varepsilon) + \log(1/\delta)}{\varepsilon},
\end{equation*}
then, with probability at least $1-\delta$, \smash{$\acc_{\rho} ( \pihat ) \ge 1-\varepsilon$}.
$\Vlearn$ queries the outcome verifier oracle $\calV$, and generates CoTs from $\piref$, no more than $m = \calO ( n \Cseq \log(n\Cseq/\delta))$ times.
\end{proposition}
\begin{proof}
The proof of this result is deferred to \Cref{subsec:seq-coverage-proof}.
\end{proof}

\paragraph{Proof sketch.} $\Vlearn$ (\Cref{alg:Vlearn}) proceeds by drawing a set of $m=\calOtilde (\Cseq)$ responses on each prompt $\bx_i \in \Dprompt{}$ in the dataset, and filtering out the responses which are either, $(a)$ duplicates, $(b)$ have low probability under $\piref (\cdot|\bx_i)$, or $(c)$ terminate in an incorrect answer. For any such prompt $\bx_i$, as long as $m$ is sufficiently large, with high probability the learner can guarantee that the set of surviving responses, $\widetilde{D}_i$, exactly equals the set,
\begin{equation}
    \calY^\star (\bx) = \{ \by \in \calY : \calV (\bx, y_T) = 1 \text{ and } \piref (\by|\bx) \ge \Cseq^{-1} \}.
\end{equation}
That is, $\calY^\star(\bx)$ is the set of all responses that have high probability under $\piref(\cdot|\bx)$ and terminate in a correct answer. Consequently, under the same high probability event, the loss optimized by $\Vlearn$ equals,
\begin{equation} \label{eq:loss-Vlearn}
    \calL (\pi ; \Dprompt{}) = \frac{1}{n} \sum_{i=1}^n \mathbb{I} \big( \pi_{1:T} (\bx_i) \not\in \widetilde{D}_i \big) = \frac{1}{n} \sum_{i=1}^n \mathbb{I} \big( \pi_{1:T} (\bx_i) \not\in \calY^\star (\bx_i) \big)
\end{equation}
Now, the key trick to show a uniform convergence bound for this loss is that if for some model $\pi \in \Pi$, we have that $\pi_{1:T} (\bx_i) \not\in \calY^\star (\bx_i)$, then this event is witnessed by the first deviation point $t$, such that $\pi_{1:t} (\bx_i) = \by_{1:t}$ for some $\by \in \calY^\star (\bx_i)$, but  $\pi_{t+1} (\bx_i) \ne y_{t+1}$. Since there can be at most $T$ such first deviation points, and each such deviation is realized by $\{ \pi (\bx_i, \by_{1:t}) \ne y_{t+1} \}$\footnote{Here, we abuse notation to let $\pi (\bx_i,\by_{1:t})$ denote the deterministic next-token $\pi$ would generate on the input sequence $(\bx_i,\by_{1:t})$} for $\pi \in \Pi$, a counting argument based on the Sauer-Shelah lemma shows there are at most $(e n \Cseq T )^d$ possible behaviors $\Pi$ can express over the set $\big( \mathbb{I} \big( \pi_{1:T} (\bx_i) \not\in \calY^\star (\bx_i) \big) : i \in [n] \big)$. This results in a bound on the growth function of the loss class corresponding to the empirical risk in \cref{eq:loss-Vlearn}, and thereby a uniform concentration bound for the same, which can be used to prove the statement of the theorem.

\begin{remark}[Partial coverage]
In \Cref{subsec:seq-coverage-proof}, we prove a more general version of \Cref{theorem:seq-coverage} when $\piref$ satisfies coverage on all but a small mass of prompts~\citep{song2024importance,chen2025coverage}. Namely, for some $\eta>0$, $\Pr_{\bx \sim \rho} \big( \piref (\by_\bx^\star | \bx) \ge \Cseq^{-1} \big) \ge 1 - \eta$ where $\by_\bx^\star = \pi_{1:T}^\star (\bx)$. The implications of this guarantee for weaker notions such as $L_1$-coverage are discussed in the appendix.
\end{remark}

\begingroup
\setlength{\algomargin}{\dimexpr\algomargin+\algonuminset\relax}
\normalem
\begin{algorithm}[t]
\caption{\texorpdfstring{$\Vlearn (\Dprompt{} \| \varepsilon,\delta,T)$}{V-Learn}}
\label{alg:Vlearn}
\noindent\hspace*{-\algonuminset}
\begin{minipage}{\dimexpr\linewidth+\algonuminset\relax}

\LinesNumbered
\setcounter{AlgoLine}{0}
\DontPrintSemicolon

{\color{gray} \# Sharpening a reference model $\piref$ satisfying sequence-level coverage using verifier feedback.}

\nl \KwIn{Class of models $\Pi$; target accuracy $1-\varepsilon$ and failure prob. $\delta$; reasoning steps $T$;\;

\hspace{2.9em} Reference model, $\piref$ satisfying $\Cseq$ sequence-level coverage (\Cref{def:seq-cov});\;

\hspace{2.9em} Prompt dataset, $\Dprompt{} = \{ \bx_i \}_{i=1}^{n}$ where $\bx_i \sim \rho$;\;

\hspace{2.9em} Outcome verifier oracle $\calV : \calX \times \Sigma \to \{ 0,1 \}$;
}

\BlankLine
\textbf{Initialize:} $m = \Cseq \log (4n \Cseq/\delta)$
\BlankLine

\For{$\bx_i \in \Dprompt{}$}{
Draw $m$ length-$T$ chains-of-thought from $\piref (\cdot |\bx_i)$ and deduplicate. Denote this set as $D_i$.\;

Discard low-probability chains-of-thought which lead to incorrect answers,
\vspace{-0.5em}
\begin{equation*}
    \widetilde{D}_i = \big\{ \by \in D_i : \calV ( \bx_i, y_T) = 1 \text{ and } \piref (\by|\bx_i) \ge \Cseq^{-1} \big\}
\end{equation*}
\vspace{-1.5em}
}

\textbf{Return:} $\pihat \in \underset{\pi \in \Pi}{\operatorname{argmin}} \ \calL (\pi ; \Dprompt{})$ for $\calL (\pi ; \Dprompt{}) = \frac{1}{n} \sum_{i=1}^n \mathbb{I} \big( \pi_{1:T} (\bx_i) \not\in \widetilde{D}_i \big)$\; \label{alg:Vlearn:7}

\end{minipage}
\end{algorithm}
\ULforem
\endgroup

\noindent While the computational cost of $\Vlearn$ scales linearly with the coverage coefficient as $\calOtilde (n \Cseq)$, the sample complexity only scales \textit{logarithmically} with $\Cseq$. The former is a necessary cost: the learner must generate \smash{$\calOtilde(\Cseq)$} traces per prompt from $\piref$ to see even a single correct CoT. 

\subsection{Main Result: Autocurriculum Decouples Coverage from Accuracy}

{Having established that learning is tractable without a curriculum, we now ask: \textit{can autocurriculum further reduce the computational cost?} We show that the answer is yes. $\Vboost$ (\Cref{alg:Vboost}) applies the same boosting-based autocurriculum from \Cref{sec:distill}, using $\Vlearn$ as the base learner. As in the SFT setting, the verifier identifies prompts where the current ensemble errs, and the base learner is retrained on those prompts. The key difference is that instead of querying a teacher for CoTs, each invocation of $\Vlearn$ generates its own traces from $\piref$ via rejection sampling.}

\begin{theorem}[Autocurriculum decouples coverage from accuracy] \label{theorem:Vboost}
Consider any target error $\varepsilon \in (0,1)$ and failure probability $\delta \in (0,1)$. Suppose $\Pi$ is deterministic, $\piref$ satisfies sequence-level coverage with parameter $\Cseq$ (\Cref{def:seq-cov}), answers are unique (\Cref{assump:sparse}) and that the optimal model is realizable (\Cref{assump:teacher}). $\Vboost$ (\Cref{alg:Vboost}) guarantees that as long as the size of the prompt dataset is at least,
\begin{equation*}
    n \ge \frac{d}{\varepsilon} \cdot \polylog(\Cseq,\varepsilon^{-1},\delta^{-1},T)
\end{equation*}
the resulting outcome-level model $\fhat$ satisfies \smash{$\acc_{\rho} \big( \fhat \big) \ge 1-\varepsilon$} with probability at least $1-\delta$. Furthermore, up to $\polylog (\Cseq, \varepsilon^{-1},\delta^{-1},T)$ factors,
\begin{itemize}[leftmargin=*]
    \item The number of length-$T$ CoTs generated from any model ($\piref$, or rollouts of the learner's own models during training), $\ncomp$ is at most $\calOtilde \big( d \Cseq + \frac{d}{\varepsilon} \big)$.
    \item The number of calls to the outcome verifier is also upper bounded by $\calOtilde \big( d \Cseq + \frac{d}{\varepsilon} \big)$.
\end{itemize}
\end{theorem}
\begin{proof}
The proof of this result is deferred to \Cref{subsec:Vboost-proof}.
\end{proof}

\noindent Compared to $\Vlearn$, which requires $\calOtilde(d\Cseq/\varepsilon)$ total traces, $\Vboost$ reduces this to $\calOtilde(d\Cseq + d/\varepsilon)$. The coverage-dependent cost $\calOtilde(d\Cseq)$ becomes a one-time burn-in that grows only logarithmically with $1/\varepsilon$; as the target accuracy increases, the learner pays as if coverage were $O(1)$. The key reason is that each of the $O(\log(1/\varepsilon))$ weak learners trains on only $\calOtilde(d)$ prompts, so the per-prompt generation cost of $\calOtilde(\Cseq)$ does not multiply with $1/\varepsilon$. The remaining $\calOtilde(d/\varepsilon)$ cost comes from evaluating the ensemble to route prompts, and is nearly independent of $\Cseq$.

\begin{algorithm}[t]
\caption{\texorpdfstring{$\Vboost (\Dprompt{} \| \Alg_{\RL}, \varepsilon, \delta, T)$}{V-Boost}}\label{alg:Vboost}
\begin{algorithmic}[1]
\State \textbf{Input:} Class of models $\Pi$, reasoning steps $T$, target failure prob. $\delta$,
\Statex \hspace{2.9em} Prompt dataset, $\Dprompt{} = \{ \bx_i \}_{i=1}^{n}$ where $\bx_i \sim \rho$,
\Statex \hspace{2.9em} Outcome verifier, $\calV : \calX \times \Sigma \to \{ 0,1 \}$,
\State \textbf{Return:} $\fhat = \Genboost ( \Dprompt{} \| \Alg_{\RL}, \varepsilon, \delta, T)$, where,
\begin{equation*}
    \Alg_{\RL} (\cdot \| \varepsilon',\delta',T) \gets \Vlearn (\cdot \| \varepsilon', \delta',T)
\end{equation*}
\end{algorithmic}
\end{algorithm}

\section{Discussion}

The central mechanism behind our results is simple: autocurriculum routes expensive supervision toward prompts the model currently gets wrong, avoiding wasted effort on already-solved ones. For SFT, the number of teacher demonstrations drops from $\calOtilde(d/\varepsilon)$ to $\calOtilde(d)$. For RLVR, the cost of rejection sampling from $\piref$ to generate correct chains-of-thought is confined to a burn-in cost, reducing total compute from $\calOtilde(d\Cseq/\varepsilon)$ to $\calOtilde(d\Cseq + d/\varepsilon)$.

\smallskip
\noindent This principle already appears, in heuristic form, in several practical systems. DAPO~\citep{yu2025dapo} filters prompts with all-correct or all-incorrect rollouts, retaining only those with mixed signals; PCL~\citep{gao2025prompt} uses a value model to select prompts at intermediate difficulty. Both can be viewed as heuristic instantiations of the verifier-guided filtering in $\Genboost$ (\Cref{alg:CoTboost}). Similarly, the generate--filter--retrain loop in ReST~\citep{gulcehre2023reinforced} and in the distillation stage of DeepSeek-R1~\citep{guo2025deepseek} is essentially the rejection sampling mechanism of $\Vlearn$ (\Cref{alg:Vlearn}); our theory quantifies when wrapping this with an outer curriculum ($\Vboost$, \Cref{alg:Vboost}) reduces total compute.\loose

\paragraph{Limitations and open directions.}
Our framework makes several simplifying assumptions that suggest natural directions for future work.

\smallskip
\noindent \textit{Self-play and iterated self-improvement.} Our RLVR results assume a fixed reference model $\piref$ with a given coverage coefficient $\Cseq$. In practice, methods like ReST~\citep{gulcehre2023reinforced} and expert iteration use the \textit{current} model as the reference, iteratively improving coverage. Understanding how autocurriculum interacts with improving coverage across rounds, and whether the burn-in cost $\calOtilde(d\Cseq)$ can be further reduced through self-play, is an important open question.

\smallskip
\noindent \textit{Online RL and policy gradient methods.} Our algorithms operate in a batch setting: train a model on a fixed dataset, then re-evaluate. Modern RL pipelines use online policy gradient updates (PPO, GRPO) where the model is updated continuously. Extending the autocurriculum framework to this online setting, where the prompt distribution and the model co-evolve, would bring the theory closer to practice.

\smallskip
\noindent \textit{Beyond coverage.} The RLVR results require $\piref$ to cover the optimal CoT with probability at least $\Cseq^{-1}$ on every prompt. When coverage fails, e.g., for prompts well beyond the model's current capabilities---our filtering approach simply gives up on these prompts. Whether autocurriculum can enable \textit{coverage expansion}, where training on easier problems builds coverage on harder ones~\citep{lee2025self,motwani2025h1}, is a natural question, closely related to explicit length-based curricula~\citep{setlur2025e3}.

\noindent In the forthcoming Part II of this paper, we demonstrate a formal sense in which autocurricula can enable this kind of coverage expansion.

\smallskip
\noindent \textit{Imperfect verification.} Our framework assumes access to a perfect outcome verifier, which is natural for domains with verifiable rewards (math, code), but extending the theory to noisy or learned reward models is an important open problem.\loose

\newpage

\bibliographystyle{plainnat}

\bibliography{refs}

\newpage

\appendix

\startcontents[apx]

\section*{Appendices}
\printcontents[apx]{}{1}{\setcounter{tocdepth}{3}}

\section{Related Work}

\paragraph{Curriculum in deep learning and language model post-training.}
Curriculum learning formalizes the intuition that presenting training examples in a meaningful order, such as from easy to hard, can improve optimization and generalization~\citep{bengio2009Curriculum}. Self-paced learning makes this idea algorithmic by alternating between selecting or reweighting examples based on a difficulty proxy and updating parameters~\citep{kumar2010SelfPaced, meng2015Objective, hacohen2019Power, fan2018Learning}.

\smallskip
\noindent \textit{Synthetic data generation and prompt curation.} Closely related to approaches like expert iteration~\citep{wang2023self,gulcehre2023reinforced,lin2025goedel}, models are trained on synthetically generated data at the cutting edge of their capabilities to optimize training signals~\citep{xu2024wizardlm,motwani2025h1,poesia2024learning}. In RL settings, curriculum strategies based on prompt curation have been shown to enable models to make progress, \textit{even when the initial pass@k performance on the target task is close to $0$}~\citep{lee2025self,prakash2025can}.

\smallskip
\noindent \textit{Autocurricula: Dynamic data selection and compute allocation.}
Autocurriculum methods, where sampling distributions are adaptively altered along the course of training, have gained traction in LLM training as a way to improve sample efficiency and stability by allowing models to guide their own data collection. In particular, self-evolving curricula, where the learner prioritizes examples it finds challenging, have shown substantial improvements in compute efficiency~\citep{yu2025dapo,khatri2025art}. One line of work focuses on adaptively selecting prompts so that models focus on tasks which are hard, but within the horizon of their capabilities~\citep{gao2025prompt,xiong2025reinforce}. Of note are approaches based on length-based curricula~\citep{setlur2025e3} and adaptive compute allocation~\citep{qu2025optimizing}, where models are gradually exposed to longer thinking budgets over the course of training. In other work, adaptive verifiable environments provide an environment-based autocurriculum for scaling RL while maintaining reliable feedback~\citep{zeng2025rlve}. Our results complement these empirical results by providing a theoretical account of how and when adaptive curriculum mechanisms can yield provable reductions in sample and compute complexity.

\paragraph{Theoretical frameworks for reasoning in LLMs.} 
Several recent theory papers provide foundations for analyzing language-model post-training and reasoning. \citet{joshi2025theory} analyze learnability in autoregressive models using chain-of-thought structure. In a related line of work, \citet{foster2024behavior} provide learning-theoretic analyses of imitation learning characterizing the effect of the sequence-length of the problem. In LLM post-training, coverage has emerged as a fundamental quantity characterizing sample and computational complexity. \citet{chen2025coverage} study the coverage properties that emerge from pre-training language models, while \citet{huang2025best} and \citet{foster2025good} show that it tightly quantifies post-training performance. The coverage of the reference model has also been used in the analysis of various training-time and inference-time algorithms for post-training LLMs
\citep{zhu2023principled,song2024importance,rohatgi2025taming}.\loose

\paragraph{Boosting and learning from counterexample queries.} Our algorithmic approaches connect to classical results for boosting, which formalize how adaptive data collection can hasten learning~\citep{Freund,freund1997decision,freund1999short}. Boosting-by-filtering is an aggregation scheme for classification, where models are trained iteratively on a sequence of evolving distributions to correct the mistakes of the previous ones, in a manner such that the aggregated model has low error. Similarly, our work is also closely connected to the problem of learning from counterexample and equivalence oracles~\citep{angluin1987learning,pmlr-v76-angluin17a,NEURIPS2021_ae06fbdc}, and subsequent generalizations, such as robust interactive learning~\citep{pmlr-v23-balcan12c,emamjomeh2017general}.

\section{Proofs for Main Results}

\subsection{Autocurriculum for SFT (Deterministic $\Pi$): Proof of \Cref{theorem:CoTboost}} \label{sec:CoTboost}

In this section, we prove \Cref{theorem:CoTboost} for the setting where the $\Pi$ is composed of deterministic models and the outcome verifier is sparse (\Cref{assump:sparse}). First we plot the evolution of $\alpha^{j,k}_r$ as a function of $r$ across different values of $j$.

\begingroup
\normalem
\renewcommand*{\subfigurelabel}[1]{}
\begin{figure}[h!]
    \centering
    \subfigure{
        \includegraphics[height=1.51in]{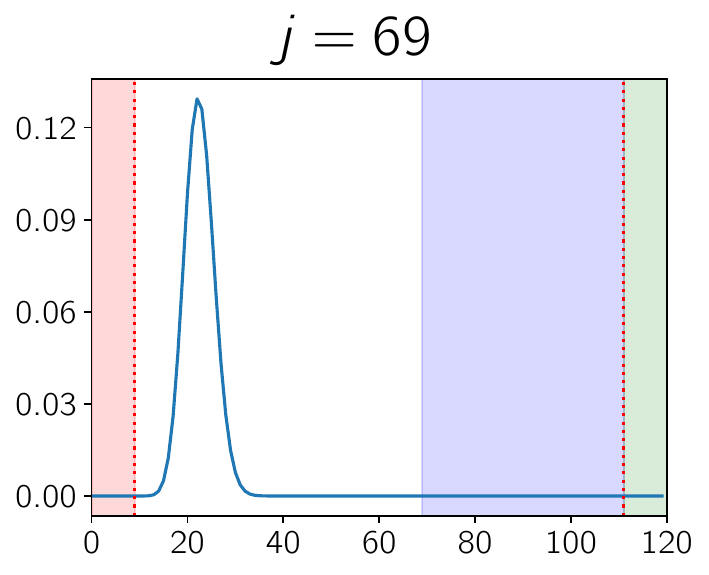}
    }\hfill
    \subfigure{
        \raisebox{-0.25mm}{
        \includegraphics[height=1.52in]{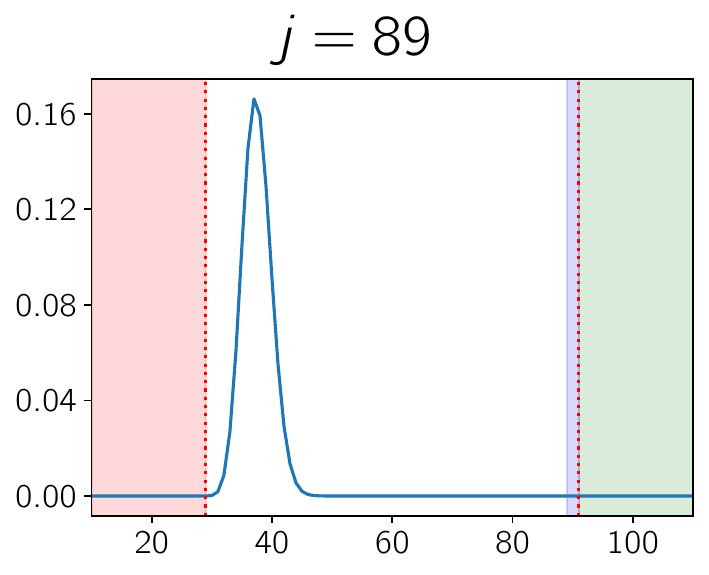}
        }
    }\hfill
    \subfigure{
        \raisebox{-0.35mm}{
        \includegraphics[height=1.51in]{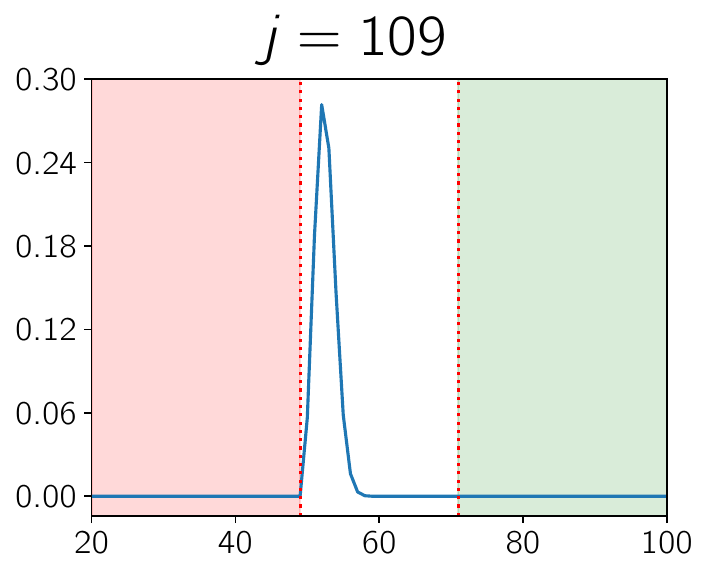}
        }
    }
    \caption{For $k=120$, we plot of $\alpha^{j,k}_r$ as a function of the rank placeholder $r$ across different values of $j$ (number of models in the current ensemble). The {\color{green!65!black} shaded green} region captures the values of the rank $r$ (for some prompt $\bx$) such that even if the remaining $k-j$ models were to all predict the wrong label on $\bx$, the plurality vote remains robustly correct. The {\color{red!80!black} shaded red} region captures the values of the rank $r$ for which even if the remaining $k-j$ models were to all predict the correct label on $\bx$, the plurality vote cannot be guaranteed to be correct. The {\color{blue!80!white} shaded blue} region plots values of the rank which are unattainable (the maximum rank achievable in iteration $j$ is $j$).}
    \label{fig:alpha}
\end{figure}
\ULforem
\endgroup

\paragraph{Proof sketch of \Cref{theorem:CoTboost}.}
The accuracy of the outcome-level model returned by $\Genboost$ can be calculated by considering the probability mass on prompts where more than half the models in the final ensemble \smash{$\Pi_k = \big\{ \pihat^0,\cdots, \pihat^{k-1} \big\}$} predict the correct label. We begin with a decomposition of the test-error incurred by \Cref{alg:CoTboost} into terms which depend on the performance of the models trained on samples from $\rho_j^\star$ in each iteration.

\begin{lemma}[Test-error decomposition] \label{lemma:CoT-breakdown}
Let $\fhat$ denote the outcome-level model returned by \Cref{alg:CoTboost}. Then,
\begin{equation*}
    \Pr_{\bx \sim \rho} \big( \fhat (\bx) \ne \pi^\star_T (\bx) \big) \le \beta^{0,k}_0 + \sum_{j=0}^{k-1}  (\err_j - \err_\star) \cdot \bbE_{\bx \sim \rho} \big[ \alpha^{j,k}_{\rank_j (\bx)} \big]
\end{equation*}
Here, recall $\err_\star = \frac{1}{4}$ and $\err_j = \Pr_{\bx \sim \rho_j^\star} \big( \pihat^j_T (\bx) \ne \pi^\star_T (\bx) \big) = 1 - \acc_{\rho_j^\star} (\pihat^j)$.
\end{lemma}
\begin{proof}
The formal proof of this result is discussed in \Cref{subsec:CoT-breakdown-proof}.
\end{proof}

\noindent \Cref{lemma:CoT-breakdown} decomposes the test error of $\fhat$ into a sum of several terms, and the proof of this result will follow by defining a potential function $\Phi_j$ which tracks the performance of the plurality of the outcome-level models, $\{ \pi_T : \pi \in \Pi_j \}$ trained until the $(j-1)^{\text{th}}$ phase,
\begin{equation*}
    \Phi_j = \bbE_{\bx \sim \rho} \big[ \beta^{j,k}_{\rank_j (\bx)} \big]
\end{equation*}
Here, the quantity $\beta^{j,k}_r$ is defined in \cref{eq:alpha}, while $\rank_j (\bx)$ is defined in \cref{eq:rank}. The term $\beta_0^{0,k}$ appearing in the statement of the above lemma is the initial value of the potential function, $\Phi_0$, while the remaining terms in the summation are precisely the differences $\Phi_j - \Phi_{j-1}$. Note that $\Phi_k$ itself takes a simple form, $\bbE_{\bx \sim \rho} \big[ \bbI ( \rank_k (\bx) \le k/2 ) \big]$, which is precisely the test error of the plurality of the ensemble induced by the models \smash{$\big\{ \pihat^{0},\cdots,\pihat^{k-1} \big\}$}, since it captures the event that $k/2$ or fewer models get the final label correct. This clarifies the connection between the $\Phi_j$'s and the final prediction error.

\smallskip
\noindent Next we discuss how to simplify \Cref{lemma:CoT-breakdown} further. It is a short calculation to show by writing down an explicit expression for $\Phi_0 = \beta^{0,k}_0$ that this term in fact decays exponentially with $k$ (a fact we prove formally in \Cref{lemma:beta00}), and so, as long as $k = \Omega (\log(1/\varepsilon))$, the first term is bounded by $\varepsilon/2$, which is within the target error guarantee of \Cref{theorem:CoTboost}. The remaining terms in \Cref{lemma:CoT-breakdown} bring out a clean tradeoff to establish: in iterations where generating training examples from $\rho_j^\star$ is easy, $\err_j$ is likely to be smaller than the threshold $\err_\star = \frac{1}{4}$, resulting in, \smash{$(\err_j - \err_\star) \cdot \bbE \big[ \alpha^{j,k}_{\rank_j(\bx)} \big] \le 0$. When sampling from $\rho_j^\star$} is hard, we will need to argue that \smash{$\alpha^{j,k}_{\rank_j(\bx)}$} is also likely to be small. The interpretation of such a result would be that, in the iterations where sampling from $\rho_j^\star$ is hard, there is sufficient leeway in the plurality of the existing models that adding an arbitrary \smash{$\pihat^j$} to the ensemble does not hurt the error significantly.

\medskip
\noindent Toward establishing such a guarantee, we will first define an event which captures whether the dataset collected in iteration $j$, \smash{$\Dout{j}$} is sufficiently large (a proxy for whether training the model in iteration $j$ is feasible),
\begin{equation} \label{eq:CoT-abort}
    \abort[j] = \big\{ |\Dout{j}| < \nprompt ( \err_\star, \delta/k, T ) \big\}
\end{equation}
where $\nprompt (\varepsilon,\delta,T)$ is the sample complexity of the base learning algorithm, $\Alg_{\CoT} (\cdot \| \varepsilon,\delta,T)$, in \Cref{alg:CoTboost}, which is chosen as the learner from \Cref{theorem:prior}. We will also define,
\begin{equation} \label{eq:pj}
    p_j = \bbE_{\bx \sim \rho} \left[ \frac{w_j (\bx)}{\| w_j \|_\infty} \right] = \bbE_{\bx \sim \rho} \left[ \frac{\alpha^{j,k}_{\rank_j(\bx)}}{\alpha^{j,k}_{\max}} \right]
\end{equation}
which is the probability that the prompt $\bx$ accepted into $\Dout{j}$ in \Cref{alg:CoTboost-subsample} in the $j^{\text{th}}$ iteration. Intuitively, a larger value of $p_j$ means that we will have more prompts to train the model \smash{$\pihat^j$} on, which itself translates to a smaller probability of this iteration aborting.

\paragraph{Analyzing the test-error decomposition of \Cref{lemma:CoT-breakdown}.}
Our main argument to analyze the main summation $\sum_{j=0}^{k-1} (\err_j - \err_\star) \cdot \bbE_{\bx \sim \rho} \big[ \alpha^{j,k}_{\rank_j(\bx)} \big]$ in \Cref{lemma:CoT-breakdown} will be to show that,
\begin{enumerate}
    \item If iteration $j$ does not abort, then with high probability, $\err_j \le \err_\star$ (\Cref{lemma:CoT-weak-learner}) and by extension, \smash{$(\err_j - \err_\star) \cdot \bbE_{\bx \sim \rho} \big[ \alpha^{j,k}_{\rank_j(\bx)} \big] \le 0$}. 
    \item If iteration $j$ aborts, then we will argue that with high probability this implies $p_j \le \calO \big( \frac{\varepsilon}{\sqrt{k}} \big)$ must be small (\Cref{lemma:CoT-truthful-abort}). Notice from its definition in \cref{eq:pj} that $p_j$ is proportional to $\bbE_{\bx \sim \rho} \big[ \alpha^{j,k}_{\rank_j(\bx)} \big]$ implying that the latter will also be small in aborted iterations. Namely,
    \begin{align*}
        (\err_j - \err_\star) \cdot \bbE_{\bx \sim \rho} \big[ \alpha^{j,k}_{\rank_j(\bx)} \big] &\le \bbE_{\bx \sim \rho} \big[ \alpha^{j,k}_{\rank_j(\bx)} \big] = p_j \cdot \alpha^{j,k}_{\max}
    \end{align*}
    A short calculation shows that $\sum_{j=0}^{k-1} \alpha^{j,k}_{\max} = \calO(\sqrt{k})$, and by choosing constants carefully, we can argue that the overall contribution of such terms can be bounded by $\varepsilon/2$.
\end{enumerate}
By invoking the above arguments to simplify the summations in \Cref{lemma:CoT-breakdown} and using the bound on $\beta_0^{0,k} \le \varepsilon/2$ (by sufficiently large choice of $k$), we arrive at the inequality,
\begin{align*}
    \Pr_{\bx \sim \rho} \big( \fhat (\bx) \ne \pi^\star_T (\bx) \big) \le \frac{\varepsilon}{2} + \frac{\varepsilon}{2} = \varepsilon,
\end{align*}
completing the proof sketch.

\bigskip
\noindent Having established the high-level approach, we will introduce the lemmas mentioned above. First we will argue that $\beta^{0,k}_0$ decays exponentially fast in $k$.

\begin{lemma} \label{lemma:beta00}
There exists an absolute constant $c > 0$ such that, $\beta^{0,k}_0 \le e^{-ck}$.
\end{lemma}
\begin{proof}
Consider a set of $k$ biased coins, $Z_1,\cdots,Z_k$, with probability of heads equal to $1-\err_\star = \frac{3}{4}$. We will show that \smash{$\beta^{0,k}_0$} equals the probability that at most $\frac{k}{2}$ of them come up heads. Since the expected number of heads is $\frac{3k}{4}$, by standard concentration arguments (say, Hoeffding's inequality), the statement of the lemma is established.

\smallskip
\noindent Let $S_j = \sum_{i=1}^j \bbI (Z_i = \texttt{H})$. We claim that, $\beta^{j,k}_r = \Pr \left( S_k \le \frac{k}{2} \ \middle| \ S_j = r \right)$. By definition of $\beta^{j,k}_r$ this is true for $j = k$. For smaller $j$, we leave as a short exercise to the reader to show that setting \smash{$\beta^{j,k}_r$} as such, results in the equation, \smash{$\beta^{j,k}_r = \err_\star \cdot \beta^{j+1,k}_r + (1-\err_\star) \cdot \beta^{j+1,k}_{r+1}$} being satisfied inductively.
\end{proof}

\noindent Next, we will show that in any iteration $j$ which does not abort (i.e., the dataset \smash{$\Dout{j}$} is sufficiently large), with high probability the model $\pihat^j$ achieves low test error under $\rho_j^\star$.

\begin{lemma} \label{lemma:CoT-weak-learner}
Suppose $\abort[j]$ is false in iteration $j \ge 0$. Instantiating the base learner $\Alg_\CoT (\cdot \| \varepsilon,\delta,T)$ in \Cref{alg:CoTboost} as the learner in \Cref{theorem:prior}, then with probability at least $1-\frac{\delta}{k}$, the model \smash{$\pihat^j$} trained in iteration $j$ satisfies,
\begin{align} \label{eq:CoT-weak-learner}
    \Pr_{\bx \sim \rho_j^\star} \Big( \pihat^j_T (\bx) \ne \pi^\star_T (\bx) \Big) = \err_j \le \err_\star = \frac{1}{4}
\end{align}
where $\rho_j^\star$ is defined in \cref{eq:rhojstar}.
\end{lemma}
\begin{proof}
The proof of this lemma is discussed in \Cref{subsec:CoT-weak-learner-proof}.
\end{proof}

\noindent The next lemma shows that in iterations which are likely to abort (i.e., sampling from $\rho_j^\star$ is hard), $p_j$ must be small. 

\begin{lemma} \label{lemma:CoT-truthful-abort}
Let $\nprompt (\varepsilon,\delta,T)$ denote the sample complexity of the base learner $\Alg_\CoT (\cdot \| \varepsilon,\delta,T)$ in \Cref{alg:CoTboost}. Suppose, for a sufficiently large constant $C > 0$,
\begin{equation*}
    |\Dprompt{}| \ge \frac{C k^{3/2}}{\varepsilon} \times \big( \nprompt ( \err_\star, \delta/k, T ) + \log(k/\delta) \big).
\end{equation*}
Then, if $p_j \ge \frac{\varepsilon}{4 \sqrt{k}}$ in iteration $j$, then $\Pr ( \abort[j] \mid \calH_{j-1} ) \le \frac{\delta}{k}$.
\end{lemma}
\begin{proof}
The proof of this lemma is discussed in \Cref{subsec:CoT-truthful-abort-proof}.
\end{proof}

\subsubsection{Proof of \Cref{theorem:CoTboost}} \label{subsec:CoTboost-proof}

\paragraph{Bound on prediction error.} By the test-error decomposition of $\fhat$ in \Cref{lemma:CoT-breakdown}, the bound on $\beta^{0,k}_0$ in \Cref{lemma:beta00} and the choice of $k = \Omega (\log(1/\varepsilon))$,
\begin{align*}
    \Pr_{\bx \sim \rho} \big( \fhat (\bx) \ne \pi^\star_T (\bx) \big) &\le \frac{\varepsilon}{2} + \sum_{j=0}^{k-1}  (\err_j - \err_\star) \cdot \bbE_{\bx \sim \rho} \big[ \alpha^{j,k}_{\rank_j (\bx)} \big]
\end{align*}
By \Cref{lemma:CoT-truthful-abort}, in any iteration $j$ where $p_j \ge \frac{\varepsilon}{4\sqrt{k}}$, with probability $1-\frac{\delta}{k}$ this iteration does not abort. Conditioned on this event by \Cref{lemma:CoT-weak-learner}, $\err_j \le \err_\star$. Union bounding across all $k$ iterations, this implies that with probability at least $1-\delta$,
\begin{align*}
    \Pr_{\bx \sim \rho} \big( \fhat (\bx) \ne \pi^\star_T (\bx) \big) &\le \frac{\varepsilon}{2} + \sum_{j=0}^{k-1} \bbE_{\bx \sim \rho} \big[ \alpha^{j,k}_{\rank_j (\bx)} \big] \cdot \bbI \Big( p_j < \frac{\varepsilon}{4 \sqrt{k}} \Big) \\
    &\overset{(a)}{\le} \frac{\varepsilon}{2} + \sum_{j=0}^{k-1} \Big( \alpha^{j,k}_{\max} p_j \Big) \cdot \bbI \Big( p_j < \frac{\varepsilon}{4\sqrt{k}} \Big) \\
    &\le \varepsilon
\end{align*}
where $(a)$ is by definition of $p_j$, and the last inequality follows by invoking \Cref{lemma:CoT-alphamax-sum} below.

\paragraph{Bound on size of prompt dataset.} The bound on the size of the prompt dataset required to establish the above guarantee for \Cref{alg:CoTboost} is demonstrated in \Cref{lemma:CoT-truthful-abort} and is,
\begin{equation*}
    n = |\Dprompt{}| \ge \frac{d}{\varepsilon} \cdot \polylog(T, |\Sigma|, \varepsilon^{-1}, \delta^{-1} )
\end{equation*}
when we plug in the sample complexity of the base CoT supervised learner, $\nprompt (\err_\star,\delta/k,T)$, from \Cref{theorem:prior}.

\paragraph{Bound on number of $\CoT(\cdot)$ oracle queries.} The number of calls made to $\CoT$ is precisely \smash{$\sum_{j=0}^{k-1} |\Dout{j}|$}, which by the constraint on \Cref{alg:CoTboost-subsample-6} of \Cref{alg:CoTboost-subsample} gives us the upper bound $k \cdot \nprompt (\err_\star, \delta/k, T)$. Overall, this means that the number of CoT queries made by \Cref{alg:CoTboost} is upper bounded by,
\begin{equation*}
    \nCoT \le d \cdot \polylog(T,|\Sigma|,\varepsilon^{-1},\delta^{-1}).
\end{equation*}
almost surely, where we again plug in the upper bound on $\nprompt (\err_\star, \delta/k, T)$ of the base learner derived in \Cref{theorem:prior}.

\begin{lemma}[Lemma 3.9 in \cite{Freund}] \label{lemma:CoT-alphamax-sum}
For $0 \le j \le k-2$, $\alpha^{j,k}_{\max} \le \frac{1.1}{\sqrt{k-j-1}}$ and $\alpha^{k-1,k}_{\max} \le 1$. As a consequence,
\begin{equation*}
    \sum_{j=0}^{k-1} \alpha^{j,k}_{\max} \le 2 \sqrt{k}.
\end{equation*}
\end{lemma}

\subsection{Autocurriculum for SFT (General $\Pi$): Proof of \Cref{theorem:CoTboost-stochastic}} \label{app:CoTboost-stochastic}

We now prove \Cref{theorem:CoTboost-stochastic}. In contrast to \Cref{theorem:CoTboost}, we use a slightly different choice of reweighting functions to define the distributions models in the ensemble are trained on, as well as a different notion of rank of a prompt. The proof goes through the intermediate step of showing that the sequence of models \smash{$\Pi_j = \big\{ \pihat^0,\cdots,\pihat^{j-1} \big\}$} trained in \Cref{alg:CoTboost-stochastic} satisfy,
\begin{equation} \label{eq:0423091}
    \Pr_{\bx \sim \rho} \left( \sum\nolimits_{\pi \in \Pi_k} \bbI \Big( \acc_\bx (\pi_T) \ge \frac{4}{5} \Big) > \frac{3k}{4} \right) \ge 1 - \varepsilon
\end{equation}
As will be discussed in more detail later, this guarantee will suffice to argue that the outcome-level mixture model $\fhat = \frac{1}{k} \sum_{j=0}^{k-1} \pihat^j_T$ satisfies $\Pr_{\bx \sim \rho} \big( \acc_\bx \big( \fhat \big) \ge \frac{3}{5} \big) \ge 1 - \varepsilon$. In contrast to the setting where the class of models $\Pi$ is deterministic, we define an approximate notion of rank, estimated via Monte-Carlo rollouts. For a model $\pi : \calX \to \Delta_\Sigma$, define the random variable,
\begin{align} \label{eq:montecarloacc}
    \wacc_\bx (\pi) &= \frac{1}{m} \sum_{i=1}^m \bbI ( \calV ( \bx, Y_i ) = 1 ), \text{ where, } \{ Y_i \}_{i=1}^m \overset{\text{i.i.d.}}{\sim} \pi_T (\cdot|\bx),
\end{align}
and define a randomized notion of rank induced by these Monte-Carlo accuracy estimates,
\begin{equation} \label{eq:apx-rank}
    \wrank_j (\bx) = \sum\nolimits_{\pi \in \Pi_j} \bbI \Big( \wacc_\bx (\pi) \ge \frac{9}{10} \Big).
\end{equation}
We reserve the $\ \widehat{}\ $ accent to denote random variables. $\wrank$ is used to define the sequence of distributions under which \Cref{alg:CoTboost-stochastic} trains models. The threshold on the per-prompt accuracy of a model, set in $\wrank_j$ as $\frac{9}{10}$, is chosen to be higher than our target accuracy of $\frac{4}{5}$ in \cref{eq:0423091} to account for the fact that the $\wrank$ is a random estimate.

\smallskip
\noindent Next we define a modified version of the $\alpha (\cdot)$'s weights considered in \Cref{alg:CoTboost} to define the sequence of weights to be constructed. For $0 \le r \le j \le k$, define,
\begin{align} \label{eq:talpha}
    &\talpha^{j,k}_r = \tbeta^{j+1,k}_r - \tbeta^{j+1,k}_{r+1}, \text{ where, } \tbeta^{j,k}_r = \begin{cases}
        \bbI (r \le 4k/5) &\text{if } j = k \\
        \err_\star \cdot \tbeta^{j+1,k}_r + \left( 1 - \err_\star \right) \cdot \tbeta^{j+1,k}_{r+1} \quad &\text{if } j < k \\
    \end{cases},\\
    &\text{where, } \err_\star = \frac{1}{10}. \nonumber
\end{align}
Note the subtle difference compared to \cref{eq:alpha} where the threshold on the rank $r$ is changed from $\frac{k}{2}$ to $\frac{4k}{5}$ in \smash{$\beta^{j,k}_r$}. This is to ensure that we can achieve the guarantee in \Cref{theorem:CoTboost-stochastic} where the targeted threshold on the per-prompt accuracy of $\frac{3}{5}$ is a constant strictly larger than $\frac{1}{2}$. We correspondingly define the target distributions under which we carry out learning to be, $\big\{ \trho_j^\star \big\}_{j \ge 0}$,
\begin{align} \label{eq:trhojstar}
    \trho_j^\star (\bx) &\propto \rho (\bx) \cdot \bbE [ \ww_j (\bx) \mid \bx, \Pi_j ], \text{ where } \ww_j (\bx) = \talpha^{j,k}_{\wrank_j(\bx)}
\end{align}
Here, the expectation is over the randomness of $\wrank$. Finally since $\wrank$ is random, we will let the notation $\Pr_\rho (\cdot)$ (resp. $\bbE_\rho [\cdot]$) to denote probabilities (resp. expectations) marginalizing over $\bx \sim \rho$ and all other sources of randomness.

\smallskip
\noindent Similarly to \Cref{lemma:CoT-breakdown}, we begin with a decomposition of the test-error incurred by \Cref{alg:CoTboost} into terms which depend on the performance of the models trained on samples from $\trho_j^\star$ in each iteration.

\begin{lemma}[Loss decomposition] \label{lemma:CoT-breakdown-stochastic}
Let $\pihat^0,\cdots,\pihat^{k-1}$ denote the sequence of models trained within \Cref{alg:CoTboost-stochastic}. Then,
\begin{equation} \label{eq:CoT-breakdown-stochastic}
    \Pr_{\rho} \Big( \wrank_k (\bx) \le \frac{4k}{5} \ \Big| \ \Pi_k \Big) = \tbeta^{0,k}_0 + \sum_{j=0}^{k-1}  (\err_j - \err_\star) \cdot \bbE_{\rho} \Big[ \talpha^{j,k}_{\wrank_j (\bx)} \ \Big| \ \Pi_j \Big]
\end{equation}
where, $\err_\star = \frac{1}{10}$ and $\err_j \triangleq \Pr_{\trho_j^\star} \left( \wacc_\bx (\pihat^j) \le \frac{9}{10} \ \middle| \ \pihat^j \right)$.
\end{lemma}
\begin{proof}
The formal proof of this result is discussed in \Cref{subsec:CoT-breakdown-stochastic-proof}.
\end{proof}

\noindent At a high level, the argument will rely on analyzing the change in a potential function, defined as,
\begin{equation*}
    \tPhi_j = \bbE_{\rho} \Big[ \tbeta^{j,k}_{\wrank_j (\bx)} \ \Big| \ \Pi_j \Big]
\end{equation*}
And noting that, $\tPhi_k$ equals the LHS of \cref{eq:CoT-breakdown-stochastic}, while $\tPhi_0 = \tbeta^{0,k}_0$ equals the first term on the RHS. As before, in the next lemma we will bound the initial value of the potential.

\begin{lemma} \label{lemma:tbeta00}
There exists an absolute constant $c > 0$ such that, \smash{$\tbeta^{0,k}_0 \le e^{-ck}$}. Consequently, as long as $k \ge C \log(1/\varepsilon)$ for a sufficiently large constant $C > 0$, \smash{$\tbeta^{0,k}_0 \le \frac{\varepsilon}{4}$}.
\end{lemma}
\begin{proof}
Consider a set of $k$ biased coins, $Z_1,\cdots,Z_k$, with probability of heads equal to $1-\err_\star = \frac{9}{10}$. By the same argument as in the proof of \Cref{lemma:beta00}, \smash{$\tbeta^{0,k}_0$} equals the probability that at most $\frac{4k}{5}$ of the coins come up heads. Since the expected number of heads is $\frac{9k}{10}$, by Hoeffding's inequality, the statement of the lemma is established.
\end{proof}

\noindent Next, we will define an event which captures whether the dataset collected in iteration $j$, \smash{$\Dout{j}$} is sufficiently large (a proxy for whether training the model in iteration $j$ is feasible),
\begin{equation} \label{eq:CoT-abort-stochastic}
    \tabort[j] = \big\{ |\Dout{j}| < \nprompt ( 1/400, \delta/k, T ) \big\}
\end{equation}
where $\nprompt (\varepsilon',\delta',T)$ is the sample complexity of the base learning algorithm, $\Alg (\cdot \| \varepsilon',\delta',\CoT)$ used in \Cref{alg:CoTboost-stochastic}. Furthermore, we will also define,
\begin{equation} \label{eq:tpj}
    \tp_j = \frac{\bbE_{\rho} \Big[ \talpha^{j,k}_{\wrank_j(\bx)} \ \Big| \ \Pi_j \Big]}{\talpha^{j,k}_{\max}}
\end{equation}
which will turn out to equal the probability that a prompt $\bx \sim \rho$ is accepted into the dataset $\Dout{j}$ in \Cref{alg:CoTboost-subsample} in the $j^{\text{th}}$ iteration. Intuitively, a larger value of $\tp_j$ means that we will have more prompts to train the model \smash{$\pihat^j$} on, which itself translates to a smaller probability of this iteration aborting.

\begin{lemma} \label{lemma:CoT-weak-learner-stochastic}
Suppose $\tabort[j]$ is false in iteration $j \ge 0$. Then, instantiating the base learning algorithm $\Alg (\cdot \| \varepsilon,\delta,\CoT)$ in \Cref{alg:CoTboost-stochastic} as the learner in \Cref{theorem:prior}, then the model \smash{$\pihat^j$} trained in iteration $j$ satisfies with probability at least $1-\frac{\delta}{k}$,
\begin{align}
    \Pr_{\trho^\star_j} \left( \wacc_\bx \big( \pihat^j \big) < \frac{9}{10} \right) \triangleq \err_j \le \frac{1}{10}.
\end{align}
\end{lemma}
\begin{proof}
The proof of this lemma is discussed in \Cref{subsec:CoT-weak-learner-stochastic-proof}.
\end{proof}

\noindent Finally, similar to \Cref{lemma:CoT-truthful-abort}, we will show that any iteration $j$ is unlikely to abort as long as $\tp_j$ is sufficiently large. This will use the fact that the marginal probability that $\bx \sim \rho$ is accepted into \smash{$\Dout{j}$} via rejection sampling equals $\tp_j$.

\begin{lemma} \label{lemma:CoT-truthful-abort-stochastic}
Suppose, for a sufficiently large constant $C > 0$,
\begin{equation*}
    |\Dprompt{}| \ge \frac{C k^{3/2}}{\varepsilon} \times \big( \nprompt ( 1/400, \delta/k, T ) + \log(k/\delta) \big).
\end{equation*}
Then, if $\tp_j \ge \frac{\varepsilon}{16 \sqrt{k}}$ in iteration $j$, then $\Pr ( \tabort[j] \mid \calH_{j-1} ) \le \frac{\delta}{k}$.
\end{lemma}
\begin{proof}
The proof of this lemma is discussed in \Cref{subsec:CoT-truthful-abort-stochastic-proof}.
\end{proof}

\noindent Finally, we introduce a lemma showing how to translate between bounds on the accuracy, and the Monte Carlo estimate of the accuracy (\cref{eq:montecarloacc}) with some slack.

\begin{lemma} \label{lemma:acc-to-apx-acc}
For the ensemble of models $\Pi_k$, we have,
\begin{equation*}
    \Pr_{\bx \sim \rho} \left( \sum\nolimits_{\pi \in \Pi_k} \bbI \Big( \acc_\bx (\pi) \ge \frac{4}{5} \Big) \le \frac{3k}{4} \ \middle| \ \Pi_k \right) \le 2 \cdot \Pr_{\rho} \left( \sum\nolimits_{\pi \in \Pi_k} \bbI \Big( \wacc_\bx (\pi) \ge \frac{9}{10} \Big) \le \frac{4k}{5} \ \middle| \ \Pi_k \right)
\end{equation*}
\end{lemma}
\begin{proof}
The proof of this lemma is discussed in \cref{subsec:acc-to-apx-acc-proof}.
\end{proof}

\noindent Having introduced all the necessary results, we are ready to prove the main result of this section.

\subsubsection{\texorpdfstring{Proof of \Cref{theorem:CoTboost-stochastic}}{Proof of Theorem~\ref{theorem:CoTboost-stochastic}}} \label{sec:CoTboost-stochastic-proof}

\paragraph{Bound on prediction error.} By the test-error decomposition of $\fhat$ in \Cref{lemma:CoT-breakdown-stochastic}, the bound on $\tbeta^{0,k}_0$ in \Cref{lemma:tbeta00} and the choice of $k = \Omega ( \log(1/\varepsilon) )$,
\begin{align*}
    \Pr_{\rho} \left( \wrank_k (\bx) \le \frac{4k}{5} \ \middle| \ \Pi_k \right) &\le \frac{\varepsilon}{4} + \sum_{j=0}^{k-1}  (\err_j - \err_\star) \cdot \bbE_\rho \Big[ \talpha^{j,k}_{\wrank_j (\bx)} \ \Big| \ \Pi_j \Big].
\end{align*}
In any iteration $j$ where \smash{$\tp_j \ge \frac{\varepsilon}{16\sqrt{k}}$}, with probability $1-\frac{\delta}{k}$ this iteration does not abort. Conditioned on this event, by \Cref{lemma:CoT-weak-learner-stochastic}, with probability $1 - \frac{\delta}{k}$, $\err_j \le \err_\star$. This implies, with probability at least $1-\delta$,
\begin{align}
    \Pr_{\rho} \left( \wrank_k (\bx) \le \frac{4k}{5} \ \middle| \ \Pi_k \right) &\le \frac{\varepsilon}{4} + \sum_{j=0}^{k-1} \bbE_{\rho} \Big[ \talpha^{j,k}_{\wrank_j (\bx)} \ \Big| \ \Pi_j \Big] \cdot \bbI \Big( \tp_j < \frac{\varepsilon}{16\sqrt{k}} \Big) \nonumber \\
    &\overset{(a)}{\le} \frac{\varepsilon}{4} + \sum_{j=0}^{k-1} \Big( \talpha^{j,k}_{\max} \tp_j \Big) \cdot \bbI \Big( \tp_j < \frac{\varepsilon}{16\sqrt{k}} \Big) \nonumber \\
    &\le \frac{\varepsilon}{2}, \label{eq:0912443201}
\end{align}
where $(a)$ is by definition of $\tp_j$, and the last inequality invokes \Cref{lemma:CoT-alphamax-sum-stochastic} introduced below. Finally, we have the following sequence of inequalities to bound the performance of the outcome-level mixture model \smash{$\fhat = \frac{1}{k} \sum_{\pi \in \Pi_k} \pi_T$}. First, noting that $\acc_\bx (\fhat) = \frac{1}{k} \sum_{\pi \in \Pi_k} \acc_\bx (\pi)$, by an application of Markov's inequality,
\begin{align*}
    \Pr_{\bx \sim \rho} \left( \acc_\bx ( \fhat ) > \frac{3}{5}  \ \middle| \ \Pi_k \right) &= \Pr_{\bx \sim \rho} \left( \sum\nolimits_{\pi \in \Pi_k} \acc_\bx ( \pi ) > \frac{3k}{5}  \ \middle| \ \Pi_k \right) \\
    &\ge \Pr_{\bx \sim \rho} \left( \sum\nolimits_{\pi \in \Pi_k} \bbI \Big( \acc_\bx (\pi) \ge \frac{4}{5} \Big) > \frac{3k}{4} \ \middle| \ \Pi_k \right)
\end{align*}
Taking the complement on both sides and invoking \Cref{lemma:acc-to-apx-acc}, we get,
\begin{align*}
    \Pr_{\bx \sim \rho} \left( \acc_\bx ( \fhat ) \le \frac{3}{5}  \ \middle| \ \Pi_k \right) &\le 2 \cdot \Pr_{\rho} \left( \wrank_k (\bx) \le \frac{4k}{5} \ \middle| \ \Pi_k \right).
\end{align*}
Plugging in the upper bound on the RHS from \cref{eq:0912443201} completes the proof.

\paragraph{Bound on size of prompt dataset.} The bound on the size of the prompt dataset required to establish the above guarantee for \Cref{alg:CoTboost-stochastic} is demonstrated in \Cref{lemma:CoT-truthful-abort-stochastic} and is,
\begin{equation*}
    n = |\Dprompt{}| \gtrsim \frac{\log(|\Pi|)}{\varepsilon} \cdot \polylog(T, \varepsilon^{-1}, \delta^{-1} )
\end{equation*}
when we plug in the sample complexity $\nprompt (1/400,\delta/k,T)$ from \Cref{theorem:prior}.

\paragraph{Bound on number of $\CoT (\cdot)$ oracle queries.} The number of calls made to $\CoT (\cdot)$ is precisely \smash{$\sum_{j=0}^{k-1} |\Dout{j}|$}, which by the constraint on \Cref{alg:CoTboost-subsample-6} of \Cref{alg:CoTboost-subsample} gives us the upper bound $k \cdot \nprompt (1/400, \delta/k, T)$. Overall, this means that the number of CoT queries made by \Cref{alg:CoTboost-stochastic} is a.s. upper bounded by,
\begin{equation*}
    \nCoT \le \log (|\Pi|) \cdot \polylog(T,\varepsilon^{-1},\delta^{-1}).
\end{equation*}

\begin{lemma} \label{lemma:CoT-alphamax-sum-stochastic}
For $0 \le j \le k-2$, $\talpha^{j,k}_{\max} \le \frac{2}{\sqrt{k-j-1}}$ and $\talpha^{k-1,k}_{\max} \le 1$. Consequently,
\begin{equation*}
    \sum_{j=0}^{k-1} \talpha^{j,k}_{\max} \le 4 \sqrt{k}.
\end{equation*}
\end{lemma}
\begin{proof}
The proof of this lemma is described in \Cref{subsec:CoT-alphamax-sum-stochastic-proof}.
\end{proof}
\subsection{Improving a Reference Model Satisfying Coverage: Proof of \Cref{theorem:seq-coverage}} \label{subsec:seq-coverage-proof}

In this section, we state and prove a more general version of \Cref{theorem:seq-coverage} in a setting where the reference model $\piref$ is allowed to fail to satisfy sequence-level coverage on a small mass of prompts.

\begin{definition}[Partial sequence-level coverage] \label{def:partial-seqcov}
$\piref$ is said to satisfy $(\Cseq,\eta)$ partial sequence-level coverage under the prompt distribution $\rho$ if,
\begin{equation} \label{eq:partial-seqcov}
    \Pr_{\bx \sim \rho} \big( \piref (\by_\bx^\star | \bx) \ge \Cseq^{-1} \big) \ge 1 - \eta
\end{equation}
where $\by_\bx^\star = \pi^\star_{1:T} (\bx)$.
\end{definition}

\begin{theorem}[Learning with partial sequence-level coverage]
Consider any target error $\varepsilon \in (0,1)$ and failure probability $\delta \in (0,1)$. Suppose the reference model $\piref$ satisfies $(\Cseq,\eta)$ partial coverage (cf. \Cref{def:partial-seqcov}). Consider the outcome-level model $\fhat$ returned by \Cref{alg:Vlearn}. Then, there exist absolute constants $C_1, C_2 > 0$ such that as long as,
\begin{equation*}
    n \ge C_1 \frac{d \log (T \Cseq) \log(n) + \log(1/\delta)}{\varepsilon},
\end{equation*}
with probability at least $1-\delta$, $\acc_\rho \big( \fhat \big) \ge 1 - C_2(\varepsilon + \eta)$. Furthermore, \Cref{alg:Vlearn} queries the outcome verifier oracle $\calV$, and generates traces from $\piref$ no more than $\calO ( n \Cseq \log(n \Cseq/\delta))$ times.
\end{theorem}

\noindent \Cref{alg:Vlearn} trains an outcome-level model by minimizing the loss $\calL (\pi ; \Dprompt{})$ defined below,
\begin{equation} \label{eq:L-loss}
    \pihat \in \argmin_{\pi \in \Pi} \calL (\pi ; \Dprompt{}), \text{ where }  \calL (\pi ; \Dprompt{}) = \frac{1}{n} \sum_{i=1}^n \bbI \big( \pi_{1:T} (\bx_i) \not\in \widetilde{D}_i \big),
\end{equation}
where the $\widetilde{D}_i$ datasets are constructed within the algorithm. For any $\bx \in \calX$, let,
\begin{equation*}
    \calY^\star (\bx) = \big\{ \by \in \calY : \calV (\bx, y_T) = 1 \text{ and } \piref ( \by | \bx) \ge \Cseq^{-1} \big\}
\end{equation*}
be the set of all sufficiently high probability strings under $\piref (\cdot |\bx)$ that also predict the correct answer on $\bx$. By the partial sequence-level coverage assumption on $\piref$,
\begin{equation} \label{eq:Yxstar-nonempty}
    \Pr_{\bx \sim \rho} \big( \calY^\star (\bx) \ne \emptyset \big) \ge 1 - \eta.
\end{equation}
We first establish a high-probability event that holds when \Cref{alg:Vlearn} is run.

\begin{lemma} \label{lemma:good-event}
Let $\calE$ denote the event $\frac{1}{n}\sum_{i = 1}^n \bbI \big( \widetilde{D}_i \ne \calY^\star (\bx_i) \big) \le 2\eta+\varepsilon$. Then $\Pr (\calE) \ge 1 - \frac{\delta}{2}$.
\end{lemma}
\begin{proof}
For $\bx \in \calX$, let $\calY (\bx) = \big\{ \by \in \calY : \piref ( \by | \bx) \ge \Cseq^{-1} \big\}$. Let $D_{\text{sample}}$ be a set of $m = \Cseq \log(4n\Cseq/\delta)$ responses $\{ \by^1,\cdots,\by^m \}$ sampled i.i.d. from $\piref (\cdot|\bx)$, and let $D \subseteq D_{\text{sample}}$ denote the deduplicated subset of responses $\by^i$ such \smash{$\piref (\by^i|\bx) \ge \Cseq^{-1}$}. For any $\by \in \calY (\bx)$, the probability of $\by \not\in D_{\text{sample}}$ is upper bounded by,
\begin{equation*}
    \Pr \big( \by \not\in D_{\text{sample}} \ \big| \ \bx \big) \le \frac{\delta}{4n \Cseq}.
\end{equation*}
Union bounding over $\by \in \calY (\bx)$ (at most $\Cseq$ such strings) gives us the inequality,
\begin{equation} \label{eq:9989182011}
    \Pr ( D = \calY (\bx) \ \big| \ \bx \big) = \Pr \big( \calY (\bx) \subseteq D_{\text{sample}} \ \big| \ \bx \big) \ge 1 - \frac{\delta}{4n}.
\end{equation}
Let $\calX_{\text{cov}}$ denote the set of prompts $\big\{ \bx : \calY^\star (\bx) \ne \emptyset \big\}$. Then, by \cref{eq:Yxstar-nonempty}, $\Pr_{\bx \sim \rho} (\calX_{\text{cov}}) \ge 1 - \eta$. For every $\bx \in \calX_{\text{cov}}$,
\begin{equation}
    \calY (\bx) = D \implies \calY^\star (\bx) = \{ \by \in D : \calV (\bx,y_T) = 1 \} \triangleq \widetilde{D} \label{eq:998918201}
\end{equation}
Furthermore, by an application of Chernoff bound to the sum of the random variables $\frac{1}{n} \sum_{i=1}^n \bbI (\bx_i \not\in \calX_{\text{cov}})$ (which has expectation at most $\eta$ from \cref{eq:Yxstar-nonempty}), with probability at least $1-\frac{\delta}{4}$,
\begin{equation*}
    \frac{1}{n} \sum_{i=1}^n \bbI \big( \bx_i \not\in \calX_{\text{cov}} \big) \le \eta + 10\sqrt{\frac{\eta \log (4/\delta)}{n}} \overset{(i)}{\le} 2 \eta + \frac{50 \log(4/\delta)}{n} \overset{(ii)}{\le} 2 \eta + \varepsilon,
\end{equation*}
where, $(i)$ uses the AM-GM inequality, while $(ii)$ uses the lower bound on the size of the dataset $n$. Combining with \cref{eq:9989182011} and \cref{eq:998918201} completes the proof.
\end{proof}

\begin{corollary} \label{corr:loss-ub}
Under the event $\calE$ (cf. \Cref{lemma:good-event}), for every model $\pi$,
\begin{equation*}
    \left| \calL (\pi ; \Dprompt{}) - \frac{1}{n} \sum_{i=1}^n g_\pi (\bx_i) \right| \le  2 \eta + \varepsilon, \text{ where, } g_\pi (\bx) = \bbI \big( \pi_{1:T} (\bx) \not\in \calY^\star (\bx) \big)
\end{equation*}
Furthermore, $\inf_{\pi \in \Pi} \bbE_{\bx \sim \rho} \big[ g_\pi (\bx) \big] \le \eta$.
\end{corollary}
\begin{proof}
By the assertion, $\big| \bbI \big( \by \in \calY^\star (\bx_i) \big) - \bbI \big(\by \in \widetilde{D}_i \big) \big| \le \bbI \big( \calY^\star (\bx_i) \ne \widetilde{D}_i \big)$ for every $i \in [n]$ and upper bounding the latter summation via \Cref{lemma:good-event} proves the main statement of the corollary. On the other hand, $\bbE_{\bx \sim \rho} \big[ g_\pi (\bx) \big] \le \eta$ at $\pi \gets \pi^\star$ (by the assumption on partial sequence-level coverage of $\piref$ in \cref{def:partial-seqcov}).
\end{proof}

\noindent Next we will define the loss class,
\begin{equation*}
    \calG = \big\{ g_\pi (\cdot) = \bbI \big( \pi_{1:T} (\cdot) \not\in \calY^\star (\cdot) \big)  : \pi \in \Pi \big\}
\end{equation*}
associated with $\Pi$. For any $\bx \in \calX$, $g_\pi (\bx) = 0$ only if $\pi^\star_{1:T} (\bx)$ takes one of at most $\Cseq$ possible sequences, which are those in $\calY (\bx)$. On such points, the class of models $\Pi$ is forced to have ``low complexity'', suggesting that the growth function of the loss class can be bounded and enabling a uniform convergence argument to bound the variation of the loss. In order to formally prove this statement, we first introduce some notation. For a set of $n$ prompts, $D = \{ \bx_i \}_{i=1}^n \in \calX^n$, define the set of possible behaviors which losses in $\calG$ can express over these points by,
\begin{align*}
    \calB_\calG (D) = \left\{ \big( g (\bx_i) : i \in [n] \big) : g \in \calG \right\}
\end{align*}
The maximum cardinality of this set over datasets of size $n$ is the growth function of the loss class under consideration.

\begin{lemma} \label{lemma:VCG}
For any dataset $D = \{ \bx_i \}_{i=1}^n$ of size $n$, $|\calB_\calG (D)| \le (e n T |\Sigma| \Cseq)^d$. This implies that,
\begin{equation*}
    \VCdim ( \calG ) \le d \log (T |\Sigma| \Cseq).
\end{equation*}
where $\VCdim (\cdot)$ returns the VC dimension of its argument.
\end{lemma}
\begin{proof}
We will defer the proof of this lemma to \Cref{subsec:VCG-proofs}.
\end{proof}

\noindent Finally, we also use a standard generalization bound based on uniform convergence and localization (i.e., the offset trick). The proof follows by invoking \cite[Theorem 5.2]{localrademacher} and an application of the AM-GM inequality.

\begin{theorem}[Excess risk bound for ERM under $0$-$1$ loss]
Consider a dataset \smash{$D = \{ \bx_i \}_{i=1}^n \overset{\text{i.i.d.}}{\sim} \rho$}. With probability at least $1-\delta$, for all $\ghat \in \calG$,
\begin{equation*}
\bbE_{\bx \sim \rho} [ \ghat (\bx) ] \lesssim \left( \frac{1}{n} \sum_{i=1} \ghat (\bx_i) - \min_{g \in \calG} \frac{1}{n} \sum_{i=1} g (\bx_i) \right) + \inf_{g \in \calG} \bbE_{\bx \sim \rho} \big[ g (\bx) \big] + \frac{\VCdim (\calG) \log ( n /\VCdim(\calG)) + \log(1/\delta)}{n}.
\end{equation*}
\end{theorem}

\subsubsection{\texorpdfstring{Proof of \Cref{theorem:seq-coverage}}{Proof of Proposition~\ref{theorem:seq-coverage}}}

In conjunction with \Cref{corr:loss-ub} and \Cref{lemma:VCG}, for the predictor $\ghat \gets g_{\pihat}$ where $\pihat \in \Pi$ is the minimizer of $\calL (\pi ; \Dprompt{})$, with probability at least $1-\delta$,
\begin{align} \label{eq:VCloss}
    \bbE_{\bx \sim \rho} \big[ g_{\pihat} (\bx) \big] \lesssim \frac{d \log (T \Cseq) \log (n) + \log(1/\delta)}{n} + \eta + \varepsilon.
\end{align}
Since $g_\pi (\bx) = \bbI \big( \pi_{1:T} (\bx) \not\in \calY^\star (\bx) \big) \ge \bbI \big( \pi_T (\bx) \ne \pi^\star_{T} (\bx) \big)$, this implies that the LHS of \cref{eq:VCloss} is further lower bounded by $\Pr_{\bx \sim \rho} \big( \pi_T (\bx) \ne \pi^\star_T (\bx) \big)$, completing the proof.
\subsection{Autocurriculum for Fine-Tuning a Reference Model: Proof of \Cref{theorem:Vboost}} \label{subsec:Vboost-proof}

The proof of \Cref{theorem:Vboost} largely follows that of \Cref{theorem:CoTboost}, except where we instantiate the base learner $\Alg_\calQ$ as $\Vlearn$ (\Cref{alg:Vlearn}). This will essentially only change the weak learning guarantee of \Cref{lemma:CoT-weak-learner}. All other definitions ($\Pi_j$, $\rho_j^\star$, etc.) are kept as before.

\begin{lemma}
Suppose the event $\abort[j]$ is false in iteration $j \ge 0$. Let the base learner $\Alg_{\RL} (\cdot \| \varepsilon',\delta',T)$ in \Cref{alg:Vboost} be instantiated as $\Vlearn$ (\Cref{alg:Vlearn}). Then, the model \smash{$\pihat^j$} trained in \Cref{alg:Vboost} satisfies with probability at least $1- \frac{\delta}{k}$,
\begin{equation*}
    \Pr_{\bx \sim \rho_j^\star} \left( \pihat^j_T (\bx) \ne \pi^\star_T (\bx) \right) = \err_j \le \err_\star = \frac{1}{4}.
\end{equation*}
\end{lemma}

\noindent What remains is to analyze the sample and computational complexity of \Cref{alg:Vboost}. We begin with the cost of generating the datasets $\big\{ \Dout{j} : 0 \le j \le k-1 \big\}$.

\paragraph{Cost of generating datasets $\big\{ \Dout{j} : 0 \le j \le k-1 \big\}$.} 

\begin{enumerate}
\item \textit{Length-$T$ CoTs generated from $\piref$.} Across the $k$ invocations of $\Vlearn$, the number of length-$T$ CoTs generated from $\piref$ is \smash{$\sum_{j=0}^{k-1} m \times |\Dout{j}|$}. With the bound on $\nprompt (\varepsilon',\delta',T)$ in \Cref{theorem:seq-coverage},
\begin{align*}
    k &= \calO(\log(1/\varepsilon)) \\
    m &= \calO(\Cseq \log(n\Cseq/\delta)) \\
    |\Dout{j}| &\le \calO (d \log(T\Cseq) + \log(k/\delta)),
\end{align*}
the number of calls to $\piref$ is upper bounded by,
\begin{equation*}
    d \Cseq \cdot \polylog(\Cseq,\varepsilon^{-1},\delta^{-1},T)
\end{equation*}
\item \textit{Length-$T$ CoTs generated from learner's models.} The number of length-$T$ CoTs generated from $\pihat^j$ is at most $|\Dprompt{}|$, which implies that the total number of CoTs generated across all models trained by the learner is upper bounded by $k \times |\Dprompt{}|$,
\begin{equation*}
    \frac{d}{\varepsilon} \cdot \polylog(\Cseq,\varepsilon^{-1},\delta^{-1},T)
\end{equation*}
Bearing only polylogarithmic dependency on the sequence-level coverage $\Cseq$.
\item \textit{Number of calls to the outcome verifier $\calV$.} Across the $k$ invocations of $\Vlearn$, the outcome verifier is called once for every call to $\piref$. Furthermore, in constructing the intermediate learning distributions $\rho_j^\star$, the learner calls the verifier $m \times |\Dprompt{}|$ times. This implies that the overall number of calls to $\calV$ also scales as,
\begin{equation*}
    d \Cseq \cdot \polylog(\Cseq,\varepsilon^{-1},\delta^{-1},T) + \frac{d}{\varepsilon} \cdot \polylog(\Cseq,\varepsilon^{-1},\delta^{-1},T)
\end{equation*}
\end{enumerate}

\paragraph{Computation spent in running $\Vlearn$.}
Running $\Vlearn$ requires a single optimization call to minimize the loss in \cref{eq:loss-Vlearn} over a dataset of size $n$. $\Vboost$ makes $k = \calO(\log(1/\varepsilon))$ calls to this oracle. However, each call is on a much smaller dataset than $\Vlearn$ would require without autocurriculum. In particular, each \smash{$\pihat^j$} is trained on a prompt dataset \smash{$\Dout{j}$} of size at most \smash{$\calO ( d \log (T \Cseq) + \log(k/\delta) )$}.

\newpage
\section{Proofs for Supporting Lemmas}

\subsection[Proofs for Lemmas from \Cref{theorem:CoTboost}]{Proofs for Lemmas from \Cref{theorem:CoTboost}} \label{sec:CoTboost-proofs}

\subsubsection[Proof of \Cref{lemma:CoT-breakdown}]{Proof of \Cref{lemma:CoT-breakdown}} \label{subsec:CoT-breakdown-proof}

For $j \ge 0$, recall by definition,
\begin{align*}
    \Phi_{j+1} &= \sum_{r=0}^{j+1} \beta^{j+1,k}_r \cdot \Pr_{\bx \sim \rho} \big( \rank_{j+1} (\bx) = r \big).
\end{align*}
Note that $\rank_{j+1} (\bx) = r$ is only possible if $\rank_j (\bx) = r$ or $\rank_j (\bx) = r-1$. Likewise, if $\rank_j (\bx) = r$, then $\rank_{j+1} (\bx) \ne r \implies \rank_{j+1} (\bx) = r+1$. With this, we decompose the above expression as,
\begin{align}
    \Phi_{j+1} &= \sum_{r=0}^{j} \beta^{j+1,k}_r \cdot \Pr_{\bx \sim \rho} \big( \rank_j (\bx) = r \big) \nonumber\\
    &\qquad - \sum_{r=0}^{j} \beta^{j+1,k}_r \cdot \Pr_{\bx \sim \rho} \big( \rank_{j+1} (\bx) = r+1 \text{ and } \rank_j (\bx) = r \big) \nonumber\\
    &\qquad + \sum_{r=1}^{j+1} \beta^{j+1,k}_r \cdot \Pr_{\bx \sim \rho} \big( \rank_{j+1} (\bx) = r \text{ and } \rank_j (\bx) = r-1 \big) \nonumber\\
    &= \sum_{r=0}^{j} \beta^{j+1,k}_r \cdot \Pr_{\bx \sim \rho} \big( \rank_j (\bx) = r \big) \nonumber\\
    &\qquad - \sum_{r=0}^{j} \big( \underbrace{\beta^{j+1,k}_r - \beta^{j+1,k}_{r+1}}_{\alpha^{j,k}_r} \big) \cdot \Pr_{\bx \sim \rho} \big( \underbrace{\rank_{j+1} (\bx) = r+1 \text{ and } \rank_j (\bx) = r}_{\equiv \big\{ \rank_j (\bx) \,=\, r \text{ and } \pihat^j_T (\bx) \,=\, \pi^\star_T (\bx) \big\}} \big)  \label{eq:100101010100}
\end{align}
The second term on the RHS of the above equation can be further decomposed as,
\begin{align*}
    &\sum_{r=0}^j \alpha^{j,k}_r \cdot \Pr_{\bx \sim \rho} \big( \rank_j (\bx) = r \text{ and } \pihat^j_T (\bx) = \pi^\star_T (\bx) \big) \\
    &= \bbE_{\bx \sim \rho} \big[ \alpha^{j,k}_{\rank_j (\bx)} \cdot \bbI \big( \pihat^j_T (\bx) = \pi^\star_T (\bx) \big) \big]   \\
    &\overset{(a)}{=} \bbE_{\bx \sim \rho} \big[ \alpha^{j,k}_{\rank_j (\bx)} \big] \cdot \Pr_{\bx \sim \rho_j^\star} \big( \pihat^j_T (\bx) = \pi^\star_T (\bx) \big) \\
    &= (1-\err_j) \times \sum_{r=0}^j \alpha^{j,k}_r \cdot \Pr_{\bx \sim \rho} \big( \rank_j (\bx) = r \big)
\end{align*}
where in $(a)$, we use the definition of $\rho_j^\star$, which is the distribution proportional to $\rho (\cdot) w_j (\cdot)$. In the final equation in the sequence, we use the definition of $\err_j$ in \cref{eq:CoT-weak-learner}. Combining back with \cref{eq:100101010100}, noting that \smash{$\alpha^{j,k}_r = \beta^{j+1,k}_r - \beta^{j+1,k}_{r+1}$} and \smash{$\beta^{j,k}_r = \err_\star \cdot \beta^{j+1,k}_r + \left( 1 - \err_\star \right) \cdot \beta^{j+1,k}_{r+1}$} where $\err_\star = \frac{1}{4}$,
\begin{align*}
    \Phi_{j+1} &= \sum_{r=0}^{j} \Big( \err_\star \cdot \beta^{j+1,k}_r + \left( 1 - \err_\star \right) \cdot \beta^{j+1,k}_{r+1} \Big) \cdot \Pr_{\bx \sim \rho} \big( \rank_j (\bx) = r \big) \\
    &\qquad + (\err_j - \err_\star) \sum_{r=0}^j \alpha^{j,k}_r \cdot \Pr_{\bx \sim \rho} \big( \rank_j (\bx) = r \big) \\
    &= \sum_{r=0}^{j} \beta^{j,k}_r \cdot \Pr_{\bx \sim \rho} \big( \rank_j (\bx) = r \big) + (\err_j - \err_\star) \cdot \bbE_{\bx \sim \rho} \big[ \alpha^{j,k}_{\rank_j(\bx)} \big] \\
    &= \Phi_j + (\err_j - \err_\star) \cdot \bbE_{\bx \sim \rho} \big[ \alpha^{j,k}_{\rank_j(\bx)} \big],
\end{align*}
Summing from $j=0$ to $j=k-1$, we arrive at the equation,
\begin{equation*}
    \Phi_k = \Phi_0 + \sum_{j=0}^{k-1} (\err_j - \err_\star) \cdot \bbE_{\bx \sim \rho} \big[ \alpha^{j,k}_{\rank_j(\bx)} \big]
\end{equation*}
Finally plugging in the explicit formula for $\Phi_0$ and $\Phi_k$ completes the proof.

\subsubsection{Proof of \Cref{lemma:CoT-weak-learner}} \label{subsec:CoT-weak-learner-proof}

Recall from \Cref{theorem:prior}, that the learning algorithm $\Alg (\cdot \| \varepsilon',\delta',\CoT)$ has sample complexity $\nprompt (\varepsilon',\delta',T)$. Assuming that $\abort[j]$ is false in iteration $j$, the size of the dataset $\Dout{j}$ is larger than $\nprompt (\err_\star,\delta/k,T)$. This implies that with probability $1-\frac{\delta}{k}$ the model \smash{$\pihat^j$} trained in iteration $j$ satisfies,
\begin{align}
    \Pr_{\bx \sim \rho_j^\star} \Big( \pihat^j_T (\bx) \ne \pi^\star_T (\bx) \Big) = \err_j \le \err_\star = \frac{1}{4}.
\end{align}
This equation uses the fact that prompts in $\Dout{j}$ fed into the base learner $\Alg (\cdot \| \varepsilon',\delta',\CoT)$ are sampled from the distribution $\rho_j^\star$ (\cref{eq:rhojstar}), by rejection sampling from $\rho$.

\subsubsection{Proof of \Cref{lemma:CoT-truthful-abort}} \label{subsec:CoT-truthful-abort-proof}

By definition of $p_j$ (\cref{eq:pj}), and by the structure of the $\subsample$ subroutine (\Cref{alg:CoTboost-4} in \Cref{alg:CoTboost}), the size of the dataset \smash{$|\Dout{j}|$} which the model $\pihat^j$ trains on, can be expressed as the sum of \smash{$n' = |\Dprompt{j}| = |\Dprompt{}|/k$} i.i.d. Bernoulli random variables, each with mean $p_j$. Indeed, $p_j$ is the probability that $\bx \sim \rho$ is accepted into the dataset \smash{$\Dout{j}$}. By an application of the multiplicative Chernoff bound,
\begin{equation*}
    \Pr \left( |\Dout{j}| \le \frac{n' p_j}{2} \ \middle| \ \calH_{j-1} \right) \le \exp \left( -\frac{n' p_j}{8} \right)
\end{equation*}
By the sufficiently large choice of $|\Dprompt{}| = n' k$, when $p_j \ge \frac{\varepsilon}{4 \sqrt{k}}$, we have that,
\begin{enumerate}
    \item $n' p_j \ge 8 \log(k/\delta)$, and,
    \item $n' p_j \ge 2 \nprompt (\err_\star, \delta/k, T)$.
\end{enumerate}
Together with the definition of $\abort[j]$ in \cref{eq:CoT-abort}, these imply that,
\begin{equation*}
    \Pr ( \abort[j] \mid \calH_{j-1} ) \le \frac{\delta}{k}
\end{equation*}

\subsection[Proofs for Lemmas from \Cref{theorem:CoTboost-stochastic}]{Proofs for Lemmas from \Cref{theorem:CoTboost-stochastic}} \label{sec:CoTboost-stochastic-proofs}

\subsubsection[Proof of \Cref{lemma:CoT-breakdown-stochastic}]{Proof of \Cref{lemma:CoT-breakdown-stochastic}} \label{subsec:CoT-breakdown-stochastic-proof}

For $j \ge 0$, recall by definition,
\begin{align*}
    \tPhi_{j+1} &= \sum_{r=0}^{j+1} \tbeta^{j+1,k}_r \cdot \Pr_{\rho} \big( \wrank_{j+1} (\bx) = r \ \big| \ \Pi_{j+1} \big).
\end{align*}
Note that the condition $\wrank_{j+1} (\bx) = \sum_{\pi \in \Pi_{j+1}} \bbI \big( \wacc_\bx (\pi) \ge \frac{9}{10} \big) = r$ is only possible in one of two cases, either,
\begin{equation*}
    \sum_{\pi \in \Pi_j} \bbI \Big( \wacc_\bx (\pi) \ge \frac{9}{10} \Big) = r \text{ or }\sum_{\pi \in \Pi_j} \bbI \Big( \wacc_\bx (\pi) \ge \frac{9}{10} \Big) = r-1
\end{equation*}
Likewise, if \smash{$\wrank_j (\bx) = r$}, then $\wrank_{j+1} (\bx) = \sum_{\pi \in \Pi_{j+1}} \bbI \big( \wacc_\bx (\pi) \ge \frac{9}{10} \big)\ne r \implies \wrank_{j+1} (\bx) = r+1$. With this, we decompose the above expression as,
\begin{align}
    \tPhi_{j+1} &= \sum_{r=0}^{j} \tbeta^{j+1,k}_r \cdot \Pr_{\rho} \big( \wrank_j (\bx) = r \ \big| \ \Pi_{j+1} \big) \nonumber\\
    &\qquad - \sum_{r=0}^{j} \tbeta^{j+1,k}_r \cdot \Pr_{\rho} \big( \wrank_{j+1} (\bx) = r+1 \text{ and } \wrank_j (\bx) = r \ \big| \ \Pi_{j+1} \big) \nonumber\\
    &\qquad + \sum_{r=1}^{j+1} \tbeta^{j+1,k}_r \cdot \Pr_{\rho} \big( \wrank_{j+1} (\bx) = r \text{ and } \wrank_j (\bx) = r-1 \ \big| \ \Pi_{j+1} \big) \nonumber\\
    &= \sum_{r=0}^{j} \tbeta^{j+1,k}_r \cdot \Pr_{\rho} \big( \wrank_j (\bx) = r \ \big| \ \Pi_j \big) \nonumber\\
    &\qquad - \sum_{r=0}^{j} \big( \underbrace{\tbeta^{j+1,k}_r - \tbeta^{j+1,k}_{r+1}}_{\talpha^{j,k}_r} \big) \cdot \Pr_{\rho} \big( \underbrace{\wrank_{j+1} (\bx) = r+1 \text{ and } \wrank_j (\bx) = r}_{\equiv \big\{ \wrank_j (\bx) \,=\, r \text{ and } \wacc_\bx ( \pihat^j ) \ge \frac{9}{10} \big\}} \ \big| \ \Pi_{j+1} \big)  \label{eq:10010101010001}
\end{align}
The last equation uses the fact that $\wrank_j$ only depends on the models in $\Pi_j$. The second term on the RHS of the above equation can be further decomposed as,
\begin{align*}
    &\sum_{r=0}^j \talpha^{j,k}_r \cdot \Pr_{\rho} \Big( \wrank_j (\bx) \,=\, r \text{ and } \wacc_\bx \big( \pihat^j \big) \ge \frac{9}{10} \ \Big| \ \Pi_{j+1} \Big) \\
    &= \bbE_{\rho} \Big[ \talpha^{j,k}_{\wrank_j (\bx)} \cdot \bbI \Big( \wacc_\bx \big( \pihat^j \big) \ge \frac{9}{10} \Big) \ \Big| \ \Pi_{j+1} \Big] \\
    &\overset{(a)}{=} \bbE_\rho \Big[ \talpha^{j,k}_{\wrank_j (\bx)} \ \Big| \ \Pi_{j+1} \Big] \cdot \Pr_{\trho_j^\star} \Big( \wacc_\bx \big( \pihat^j \big) \ge \frac{9}{10} \ \Big| \ \pihat^j \Big) \\
    &\overset{(b)}{=} (1-\err_j) \times \sum_{r=0}^j \talpha^{j,k}_r \cdot \Pr_{\rho} \big( \wrank_j (\bx) = r \ \big| \ \Pi_j \big)
\end{align*}
where in $(a)$, we use several facts. Firstly that $\wrank_j$ only depends on the models in $\Pi_j$, so we can change the conditioning from $\Pi_{j+1} \to \Pi_j$. Secondly, recalling that $\trho_j^\star$ (cf. \cref{eq:trhojstar}) is the distribution satisfying \smash{$\trho_j^\star (\bx) \propto \rho (\bx) \cdot \bbE [ \tw_j (\bx) \mid \bx, \Pi_j]$}, for any (possibly randomized) test function, $g (\cdot)$, by definition of $\tw_j$,
\begin{equation*}
    \bbE_{\bx \sim \trho_j^\star} [ g (\bx) \mid g ] = \bbE_{\bx \sim \rho} [ g(\bx) \cdot \bbE [ \tw_j (\bx) \mid \bx, \Pi_j] \mid g ] = \bbE_{\rho} \Big[ g(\bx) \cdot \talpha^{j,k}_{\wrank_j (\bx)} \ \Big| \ g, \Pi_j \Big]
\end{equation*}
which is also used within $(a)$. In $(b)$ we use the definition of $\err_j$ from \cref{lemma:CoT-weak-learner-stochastic}. Combining back with \cref{eq:10010101010001}, noting that \smash{$\talpha^{j,k}_r = \tbeta^{j+1,k}_r - \tbeta^{j+1,k}_{r+1}$}, and the recursion for \smash{$\tbeta^{j,k}_r$} in \cref{eq:talpha} and rearranging,
\begin{align*}
    \tPhi_{j+1} &= \sum_{r=0}^{j} \Big( \err_\star \cdot \tbeta^{j+1,k}_r + \left( 1 - \err_\star \right) \cdot \tbeta^{j+1,k}_{r+1} \Big) \cdot \Pr_{\rho} \big( \wrank_j (\bx) = r \ \big| \ \Pi_j \big) \\
    &\qquad + (\err_j - \err_\star) \sum_{r=0}^j \talpha^{j,k}_r \cdot \Pr_{\rho} \big( \wrank_j (\bx) = r \ \big| \ \Pi_j \big) \\
    &= \sum_{r=0}^{j} \tbeta^{j,k}_r \cdot \Pr_{\rho} \big( \wrank_j (\bx) = r \ \big| \ \Pi_j \big) + (\err_j - \err_\star) \cdot \bbE_{\rho} \Big[ \talpha^{j,k}_{\wrank_j(\bx)} \ \Big| \ \Pi_j \Big] \\
    &= \tPhi_j + (\err_j - \err_\star) \cdot \bbE_{\rho} \Big[ \talpha^{j,k}_{\wrank_j(\bx)} \ \Big| \ \Pi_j \Big],
\end{align*}
Summing from $j=0$ to $j=k-1$, we arrive at the equation,
\begin{equation*}
    \tPhi_k = \tPhi_0 + \sum_{j=0}^{k-1} (\err_j - \err_\star) \cdot \bbE_{\rho} \Big[ \talpha^{j,k}_{\wrank_j(\bx)} \ \Big| \ \Pi_j \Big]
\end{equation*}
Finally plugging in the explicit formula for $\tPhi_0$ and $\tPhi_k$ completes the proof.

\subsubsection[Proof of \Cref{lemma:CoT-weak-learner-stochastic}]{Proof of \Cref{lemma:CoT-weak-learner-stochastic}} \label{subsec:CoT-weak-learner-stochastic-proof}

Recall from \Cref{theorem:prior}, that the learning rule $\Alg (\cdot \| \varepsilon',\delta',\CoT)$ has sample complexity $\nprompt (\varepsilon',\delta',T)$. Assuming that $\abort[j]$ is false in iteration $j$, the size of the dataset $\Dout{j}$ is larger than $\nprompt (1/400,\delta/k,T)$. This implies that with probability $1-\frac{\delta}{k}$ the model \smash{$\pihat^j$} trained in iteration $j$ satisfies,
\begin{align} \label{eq:trhojstaracc}
    \acc_{\trho_j^\star} \big( \pihat^j_T \big) \ge \frac{399}{400}.
\end{align}
This inequality uses the fact that prompts in the dataset $\Dout{j}$ fed into the base learner $\Alg (\cdot \| \varepsilon',\delta',\CoT)$ are drawn from the distribution $\trho_j^\star$ (\cref{eq:trhojstar}) via rejection sampling from $\rho$. By an application of Markov's inequality, \cref{eq:trhojstaracc} translates into the following guarantee on \smash{$\pihat^j$},
\begin{equation} \label{eq:1920}
    \Pr_{\bx \sim \trho_j^\star} \Big( \acc_{\bx} \big( \pihat^j \big) \ge \frac{19}{20} \Big) \ge \frac{19}{20}.
\end{equation}
Finally, we will use this to show that with some slack across both constants in the above inequality, we have,
\begin{equation*}
    \Pr_{\bx \sim \trho_j^\star} \Big( \wacc_{\bx} \big( \pihat^j \big) \ge \frac{9}{10} \Big) \ge \frac{9}{10},
\end{equation*}
which is the statement of the lemma. In order to show this, we first argue that for any $\bx \in \calX$,
\begin{equation} \label{eq:acc-wacc}
    \Pr \Big( \wacc_{\bx} \big( \pihat^j \big) \ge \frac{9}{10} \ \Big| \  \bx,\pihat^j,\ \acc_{\bx} \big( \pihat^j \big) \ge \frac{19}{20} \Big) \ge \frac{19}{20}
\end{equation}
Note that $\wacc_{\bx}$ is computed as a Monte Carlo estimate, and is thereby an average of $m$ Bernoulli random variables each having mean $\acc_\bx (\pihat^j) \ge \frac{19}{20}$. By Chernoff bound, as long as $m$ is a sufficiently large constant (which it is chosen to satisfy within \Cref{alg:CoTboost-stochastic}), \cref{eq:acc-wacc} holds with probability at least $\frac{19}{20}$. Taking an expectation on both sides of \cref{eq:acc-wacc},
\begin{equation*}
    \Pr_{\bx \sim \trho_j^\star} \Big( \wacc_{\bx} \big( \pihat^j \big) \ge \frac{9}{10} \Big) \ge \frac{19}{20} \times \Pr_{\bx \sim \trho_j^\star} \Big( \acc_{\bx} \big( \pihat^j \big) \ge \frac{19}{20} \Big) \overset{(a)}{\ge} \left( \frac{19}{20} \right)^2 \ge \frac{9}{10}
\end{equation*}
where $(a)$ follows from \cref{eq:1920}. This completes the proof.

\subsubsection[Proof of \Cref{lemma:CoT-truthful-abort-stochastic}]{Proof of \Cref{lemma:CoT-truthful-abort-stochastic}} \label{subsec:CoT-truthful-abort-stochastic-proof}

By definition of $\tp_j$ (\cref{eq:tpj}), and by the structure of the $\subsample$ subroutine (\Cref{alg:CoTboost-4} in \Cref{alg:CoTboost}), the size of the dataset \smash{$|\Dout{j}|$} which the model \smash{$\pihat^j$} is trained on, can be expressed as the sum of \smash{$n' = |\Dprompt{j}| = |\Dprompt{}|/k$} i.i.d. Bernoulli random variables, each with mean $\tp_j$. Indeed, $\tp_j$ is the probability that $\bx \sim \rho$ is accepted into the dataset \smash{$\Dout{j}$} (cf. \Cref{alg:CoTboost-subsample-7} of \Cref{alg:CoTboost-subsample}). By an application of the multiplicative Chernoff bound,
\begin{equation*}
    \Pr \left( |\Dout{j}| \le \frac{n' \tp_j}{2} \ \middle| \ \calH_{j-1} \right) \le \exp \left( -\frac{n' \tp_j}{8} \right)
\end{equation*}
By the sufficiently large choice of $|\Dprompt{}| = n'k$ in the statement of this lemma, when $\tp_j \ge \frac{\varepsilon}{16 \sqrt{k}}$,
\begin{enumerate}
    \item $n' \tp_j \ge 8 \log(k/\delta)$, and,
    \item $n' \tp_j \ge \nprompt \big( 1/400, \delta/k, T, \CoT \big)$.
\end{enumerate}
Together with the definition of $\tabort[j]$ in \cref{eq:CoT-abort-stochastic}, these inequalities imply,
\begin{equation*}
    \Pr ( \tabort[j] \mid \calH_{j-1} ) \le \frac{\delta}{k}
\end{equation*}

\subsubsection[Proof of \Cref{lemma:acc-to-apx-acc}]{Proof of \Cref{lemma:acc-to-apx-acc}} \label{subsec:acc-to-apx-acc-proof}

Let $\calE$ denote the event that $\sum\nolimits_{\pi \in \Pi_k} \bbI \big( \acc_\bx (\pi) \ge \frac{4}{5} \big) \le \frac{3k}{4}$. Let $Z = \sum\nolimits_{\pi \in \Pi_k} \bbI \big( \wacc_\bx (\pi) \ge \frac{9}{10} \big)$. Under the event $\calE$, for at most $\frac{3k}{4}$ choices of $\pi \in \Pi_k$, we have that $\acc_\bx (\pi) \ge \frac{4}{5}$. This implies that,
\begin{equation*}
    \bbE [Z \mid \bx,\Pi_k,\calE] = \bbE \Big[ \sum\nolimits_{\pi \in \Pi_k} \bbI \Big( \wacc_\bx (\pi) \ge \frac{9}{10} \Big) \ \Big| \ \bx, \Pi_k, \calE \Big] \le \frac{3k}{4} + \frac{k}{4} \times \frac{1}{10} = \frac{31k}{40}
\end{equation*}
where the last equation follows from the choice of $m$ within the Monte Carlo estimate in \Cref{alg:CoTboost-stochastic} being a sufficiently large constant, so that $\Pr ( \wacc_\bx (\pi) \ge \frac{9}{10} | \bx,f, \acc_\bx (\pi) \le \frac{4}{5} ) \le \frac{1}{10}$. 
Furthermore, note that $\wacc_\bx (\pi)$ is independent across $\pi \in \Pi_k$ conditioned on $\bx$ and $\Pi_k$, which implies that,
\begin{equation*}
    \Var [Z \mid \bx,\Pi_k,\calE] = \Var \Big[ \sum\nolimits_{\pi \in \Pi_k} \bbI \Big( \wacc_\bx (\pi) \ge \frac{9}{10} \Big) \ \Big| \ \bx, \Pi_k, \calE \Big] \le \frac{31k}{40}
\end{equation*}
Noting that the variance of a sum of independent Bernoulli random variables is upper bound by its mean. Therefore, by an application of Chebyshev's inequality,
\begin{equation*}
    \Pr \left( Z \le \bbE [Z] + 2\sqrt{\Var[Z]} \ \middle| \ \bx, \Pi_k, \calE \right) \ge \frac{1}{2}
\end{equation*}
Plugging in the upper bound on $\bbE [Z]$ and $\Var [Z]$, and choosing $k$ to be at least a sufficiently large absolute constant so that \smash{$\bbE [Z] + 2\sqrt{\Var[Z]} \le \frac{31k}{40} + 2 \sqrt{\frac{31k}{40}} \le \frac{4k}{5}$}, we have that,
\begin{equation*}
    \Pr \left( \sum\nolimits_{\pi \in \Pi_k} \bbI \Big( \wacc_\bx (f) \ge \frac{9}{10} \Big) \le \frac{4k}{5} \ \middle| \ \bx, \Pi_k, \calE \right) \ge \frac{1}{2}
\end{equation*}
Finally, multiplying both sides by $\Pr (\calE|\bx,\Pi_k)$ and taking an expectation over $\bx \sim \rho$, we have that,
\begin{equation*}
    \Pr_{\bx \sim \rho} ( \calE | \Pi_k ) \le 2 \cdot \Pr_{\rho} \left( \sum\nolimits_{\pi \in \Pi_k} \bbI \Big( \wacc_\bx (\pi) \ge \frac{9}{10} \Big) \le \frac{4k}{5} \ \middle| \ \Pi_k \right)
\end{equation*}
Completing the proof of the result.

\subsubsection[Proof of \Cref{lemma:CoT-alphamax-sum-stochastic}]{Proof of \Cref{lemma:CoT-alphamax-sum-stochastic}} \label{subsec:CoT-alphamax-sum-stochastic-proof}

Let $Z_1,\cdots,Z_k$ denote a sequence of $k$ biased coins with probability of heads equal to $1- \err_\star = \frac{9}{10}$. By following the same argument as in \Cref{lemma:beta00}, we have an explicit form for $\tbeta^{j,k}_r$ as equal to $\Pr \big(S_k \le \frac{4k}{5} \ \big| \ S_j = r \big)$ where $S_j = \sum_{i=1}^j \bbI (Z_i = \texttt{H})$. Then,
\begin{align}
    \talpha^{j,k}_{\max} &= \max_{0 \le r \le j} \big( \tbeta^{j+1,k}_{r} - \tbeta^{j+1,k}_{r+1} \big) \nonumber\\
    &\le \max_{0 \le r \le j} \Pr \Big( S_k \le \frac{4k}{5} \ \Big| \ S_{j+1} = r \Big) - \Pr \Big( S_k \le \frac{4k}{5} \ \Big| \ S_{j+1} = r+1 \Big) \nonumber\\
    &\overset{(a)}{=} \max_{0 \le r \le j} \Pr \Big( S_k - S_{j+1} \le \frac{4k}{5} - r \Big) - \Pr \Big( S_k - S_{j+1} \le \frac{4k}{5} - (r+1) \Big) \nonumber\\
    &= \max_{0 \le r \le j} \Pr \Big( S_k - S_{j+1} = \frac{4k}{5} - r \Big) \label{eq:012312}\\
    &\overset{(b)}{\le} \frac{1}{\sqrt{2 \pi \cdot \err_\star (1-\err_\star) \cdot (k-j-1)}} \nonumber\\
    &\le \frac{2}{\sqrt{k-j-1}} \nonumber
\end{align}
where $(a)$ uses the fact that $S_k - S_{j+1}$ and $S_{j+1}$ are independent, while $(b)$ uses a standard upper bound on the Binomial PMF using the Stirling approximation for $j < k-1$. When $j=k-1$, \cref{eq:012312} gives us an upper bound of $1$.
\subsection{Proofs for Lemmas from \Cref{theorem:seq-coverage}}

\subsubsection{Proof of \Cref{lemma:VCG}} 

\label{subsec:VCG-proofs}
Recall the definition, $\calY (\bx) = \big\{ \by \in \calY : \piref (\by|\bx) \ge \Cseq^{-1} \big\}$. By the pigeonhole principle, $|\calY (\bx)| \le \Cseq$. With this, we may rewrite $\calB_\calG (D)$ as,
\begin{align*}
    \calB_\calG (D) = \left\{ \Big( \mathbb{I} \big( \pi_T (\bx_i) = \pi^\star_T (\bx_i) \big) \cdot \mathbb{I} \big( \pi_{1:T} (\bx_i) \in \calY (\bx_i) \big) : i \in [n] \Big) : \pi \in \Pi \right\}
\end{align*}
Now, define the following list of tables, $\calT (D)$:
\begin{align*}
    \calT (D) = \left\{ \left( \pi(\bx_i, \by^i_{1:t-1}) : t \in [T], i \in [n], \by^i \in \calY (\bx_i) \right) : \pi \in \Pi \right\}
\end{align*}
We will prove two claims:

\begin{claim} \label{claim:1}
$|\calT (D)| \le (e n T |\Sigma| \Cseq )^d$ where $d = \Ndim (\Pi)$ is the Natarajan dimension of $\Pi$.
\end{claim}
\begin{proof}
$\calT (D)$ captures the number of ways in which $\Pi$ labels a fixed set of $n T \Cseq$ prefixes. The proof of this claim follows by a generalization of the Sauer-Shelah lemma to multiclass predictors \citep{haussler1995generalization}.
\end{proof}

\begin{claim} \label{claim:2}
For any fixed $D$, there exists a surjection from $\calT (D) \to \calB_\calG (D)$.
\end{claim}
\begin{proof}
We will argue that if for any $\pi \in \Pi$, we are given the corresponding table $\big( \pi(\bx_i, \by^i_{1:t-1}) : t \in [T], i \in [n], \by^i \in \calY (\bx_i) \big) \in \calT (D)$ and also $\{ \calY (\bx_i) : i \in [n] \}$ and $\{ \pi^\star_T (\bx_i) : i \in [n] \}$ (but there is no explicit identification of $\pi$ itself), we can compute $\mathbb{I} \big( \pi_T (\bx_i ) = \pi^\star_T (\bx_i) \big) \cdot \mathbb{I} \big( \pi_{1:T} (\bx_i) \in \calY (\bx_i) \big)$ for this $\pi$. The procedure is as follows.

\begin{itemize}
    \item First note that we can infer $y_1^i = \pi (\bx_i)$ from the table we are given. Looking at the table $\calY (\bx_i)$, we can identify if there exists a $\bz^{i,1} \in \calY (\bx_i)$ such that $\bz_1^{i,1} = y_1$.
    \item If no such $\bz^{i,1}$ exists, the procedure terminates, and we assert that $\mathbb{I} \big( \pi_{1:T} (\bx_i) \in \calY (\bx_i) \big) = 0$, since the partial CoT generated by $\pi$ does not belong to the set of prefixes realized by strings in $\calY (\bx_i)$.
    \item If some such $\bz^{i,1}$ exists, we proceed by computing $y_2^i = \pi(\bx_i,\bz_1^{i,1})$, which is also present in the table. We again check if there exists a $\bz^{i,2} \in \calY (\bx_i)$ such that $\bz_{1:2}^{i,2} = \by^i_{1:2}$. If no such $\bz^{i,2}$ exists, we terminate and return $0$. If it exists, we proceed to the next step.
    \item In any iteration $t$, if the procedure has not yet terminated, we have a candidate sequence $(_1^i,\cdots,y_{t-1}^i)$, which is inductively assumed to compute the first $t-1$ symbols of $\pi_{1:T} (\bx_i)$, and is also the prefix of some string $\bz^{i,t-1} \in \calY (\bx_i)$. We compute \smash{$y^i_t = \pi \big( \bx_i,\bz^{i,t-1}_{1:t-1} \big)$} by looking at the corresponding entry in $\calT$. If \smash{$\by^i_{1:t} = \bz^{i,t}_{1:t}$} for some $\bz^{i,t} \in \calY (\bx_i)$, we proceed to iteration $t+1$. If not, we terminate the procedure.
\end{itemize}
By the end of this process, we can compute $\pi_{1:T} (\bx_i)$ if $\pi_{1:T} (\bx_i) \in \calY (\bx_i)$, or certify that $\pi_{1:T} (\bx_i) \not\in \calY (\bx_i)$. In the former case, we may check whether $\calV (\bx_i, y) = 1$ for $y = \pi^\star_T (\bx_i)$ (which is also fixed and does not depend on the $\pi$ under consideration). Thus, regardless of which case we are in, it is possible to compute $\mathbb{I} \big( \pi_T (\bx_i ) = \pi^\star_T (\bx_i) \big) \cdot \mathbb{I} \big( \pi_{1:T} (\bx_i) \in \calY (\bx_i) \big)$ given the table in $\calT (D)$ corresponding to some $\pi$.
\end{proof}

\noindent As a consequence of \Cref{claim:1,claim:2}, we arrive at the statement $|\calB_\calG (D)| \le |\calT (D)| \le (e n T |\Sigma| \Cseq)^d$, proving the lemma.

\end{document}